\documentclass[a4paper,reqno]{amsart}

\usepackage[english]{babel}
\usepackage[utf8]{inputenc}
\usepackage{amsmath, amsthm}
\usepackage{graphicx}
\usepackage{geometry}
\usepackage{amssymb}
\usepackage{fancyhdr}
\usepackage{tikz-cd}
\usepackage{physics}
\usepackage{natbib}
\numberwithin{equation}{section}

\theoremstyle{definition}
\newtheorem{definition}{Definition}
\newtheorem{ex}{Example}
\newenvironment{example}
  {%
   \pushQED{\qed}\begin{ex}}
  {\popQED\end{ex}}
\newtheorem{remark}{Remark}

\theoremstyle{plain}
\newtheorem{theorem}{Theorem}
\newtheorem{lemma}[theorem]{Lemma}
\newtheorem{proposition}[theorem]{Proposition}
\newtheorem{corollary}[theorem]{Corollary}

\DeclareMathOperator\supp{supp}
\DeclareMathOperator\Id{Id}

\title{Homogeneous Vector Bundles and\\ $G$-equivariant Convolutional Neural Networks}
\author{Jimmy Aronsson}
\address{Chalmers University of Technology, Department of Mathematical Sciences\\
SE-412\,96 Gothenburg, Sweden}
\email{jimmyar@chalmers.se}
\date{\today}

\begin{document}

\begin{abstract}
$G$-equivariant convolutional neural networks (GCNNs) is a geometric deep learning model for data defined on a homogeneous $G$-space $\mathcal{M}$. GCNNs are designed to respect the global symmetry in $\mathcal{M}$, thereby facilitating learning. In this paper, we analyze GCNNs on homogeneous spaces $\mathcal{M} = G/K$ in the case of unimodular Lie groups $G$ and compact subgroups $K \leq G$. We demonstrate that homogeneous vector bundles is the natural setting for GCNNs. We also use reproducing kernel Hilbert spaces to obtain a precise criterion for expressing $G$-equivariant layers as convolutional layers. This criterion is then rephrased as a bandwidth criterion, leading to even stronger results for some groups.
\end{abstract}

\maketitle

\tableofcontents

\section{Introduction}

Developments in deep learning have increased dramatically in recent years. Even though multilayer perceptrons \cite{alpaydin2020introduction} and other general-architecture models work well for some tasks, achieving higher levels of performance often requires models that are more tailored to each application, and which incorporate some level of understanding of the data. \emph{Geometric deep learning} \cite{bronstein2017geometric,bronstein2021geometric,cao2020comprehensive,cohen2019gauge,monti2017geometric} is the approach of using inherent geometric structure in data, and symmetry derived from geometry, to improve deep learning models.

\emph{Convolutional neural networks (CNNs)} are among the simplest and most broadly applicable general-architecture models. They have been successfully applied to image classification and segmentation \cite{peng2017large,xu2017large,zhang2018fully}, text summarization \cite{song2019abstractive}, pose estimation \cite{mehta2017vnect}, sign language recognition \cite{koller2018deep}, and many other tasks. One reason why CNNs are so useful is that \emph{convolutional layers}, the basic building blocks of CNNs, commute with the translation operator in $\mathbb{Z}^2$; convolutional layers are \emph{translation equivariant}. In image classification tasks, for instance, $\mathbb{Z}^2$ represents the underlying pixel lattice, and translation equivariance helps CNNs identify objects in images regardless of their exact pixel coordinates. As convolutional layers respect the global translation symmetry in $\mathbb{Z}^2$, CNNs are examples of geometric deep learning models.

\emph{$G$-equivariant convolutional neural networks (GCNNs)} \cite{cohen2016group,cohen2018general} are generalizations of CNNs to data points defined on homogeneous $G$-spaces $\mathcal{M}$. Convolutional layers that commute with the action $G \times \mathcal{M} \to \mathcal{M}$ of the global symmetry group $G$, remove the need for GCNNs to learn about the global symmetry. It is already built into the network. This enables GCNNs to focus on learning other relevant features in data, potentially improving performance. One example is the detection of tumors in digital pathology. Images of tumors can have any orientation, and GCNNs with both translation and rotation equivariant layers have higher accuracy than ordinary CNNs \cite{veeling2018rotation}. Rotation equivariance is also highly useful in 3D inference problems \cite{worrall2018cubenet}, in point cloud recognition \cite{li2019discrete}, and in other tasks.

\emph{Gauge equivariant} neural networks \cite{cheng2019covariance,cohen2019gauge,favoni2020lattice,luo2021gauge} are instead designed to respect \emph{local} symmetries. For example, computations involving vector fields - in meteorology or other areas - require vectors to be expressed in components. This requires a frame; a smooth assignment of a basis to each tangent space. However, the sphere and other non-parallelizable manifolds do not admit a global frame, so the computations must be performed locally, using different local frames for different regions on the manifold. It is then important that any numerical results obtained in one frame are compatible with those obtained in any other frame on overlapping regions. In other words, the computations should be equivariant with respect to the choice of local frame, which is a viewed as a gauge degree of freedom; a local symmetry. Gauge equivariant neural networks have also been introduced for problems exhibiting other local symmetries, primarily in lattice gauge theory.

In this paper, we study the mathematical foundations of GCNNs and characterize convolutional layers in terms of more abstract layers. Our contributions are threefold:
\begin{itemize}
\item We analyze a general framework that include both gauge equivariant neural networks and GCNNs, that only differ in whether layers respect a local gauge symmetry or a global translation symmetry. Moreover, we show that GCNNs are naturally expressed in terms of homogeneous vector bundles.

\item In general, not all $G$-equivariant layers can be written as convolutional layers. We investigate the relation between these types of layers for all homogeneous spaces $\mathcal{M} = G/K$ when $G$ is a unimodular Lie group and $K \leq G$ is a compact subgroup. As a result of this investigation, we find a criterion for expressing $G$-equivariant layers as convolutional layers (Theorem \ref{thm:main}).

\item We highlight the close relationship between convolutional layers in GCNNs, reproducing kernel Hilbert spaces (RKHS), and bandwidth. We reformulate the criterion in Theorem \ref{thm:main} as a bandwidth criterion and prove that, when $G$ is discrete abelian or finite,\footnote{Discrete Lie groups are countable, in this paper, as we assume smooth manifolds to be Hausdorff and second-countable.} all $G$-equivariant layers are indeed convolutional layers (Corollaries \ref{corr:main2}-\ref{corr:finiteG}).
\end{itemize}
This work was inspired by a number of papers \cite{cheng2019covariance,cohen2016group,cohen2018spherical,cohen2018general,cohen2019gauge,weiler20183d}. The theoretical papers \cite{cheng2019covariance,cohen2018general} have been of particular importance, as our work grew from a desire to understand the mathematics of equivariant neural networks in even greater detail.

In the case of compact groups $G$, the Peter-Weyl theorem and other powerful tools have allowed researchers to study GCNNs using harmonic analysis. Among the most well-known results in this direction is Theorem 1 in \cite{kondor2018generalization}, which uses Fourier analysis on $G$ to establish that the layers in a $G$-equivariant feed-forward neural network must be generalized convolutional layers, when $G$ is compact. This result is similar to our second contribution above and we discuss the distinction in Section \ref{subsec:group}. Others have used the well-known representation theory of the compact group $G = SO(3)$ to study rotation equivariant GCNNs for spherical data \cite{esteves2020theoretical,esteves20173d,esteves2018learning}.

The paper is structured as follows. We summarize the relevant machine learning background in Section \ref{subsec:CNN}, and discuss a framework for equivariant neural networks in Section \ref{subsec:ENN}. In Section \ref{sec:GCNN}, we restrict attention to homogeneous spaces $G/K$ where $G$ is a unimodular Lie group and \mbox{$K \leq G$} is a compact subgroup. Section \ref{subsec:hombundle} explains the relation between GCNNs, homogeneous vector bundles, and induced representations. This relation is used to motivate the definition of $G$-equivariant layers in Section \ref{subsec:group}, where we also discuss convolutional layers and prove the aforementioned Theorem \ref{thm:main}; this result characterizes when a $G$-equivariant layer is a convolutional layer, in terms of RKHS. Section \ref{subsec:RKHS} then relates RKHS to bandwidth, leading to a reformulation of Theorem \ref{thm:main} (Corollary \ref{corr:main2}) as well as a few stronger results. Finally, in Section \ref{sec:discussion}, we summarize our work and end with a discussion.

\section{Foundations of equivariant neural networks}

In this section, we give an introduction to convolutional neural networks (CNNs) and discuss a simple framework for equivariant neural networks.

\subsection{Convolutional neural networks}\label{subsec:CNN} 

CNNs were first introduced in 1979 under the name of Neocognitrons, and were used to study visual pattern recognition \cite{fukushima1982neocognitron}. In the 1990s, CNNs were successfully applied to problems such as automatic recognition of handwritten digits \cite{lecun1995learning} and face recognition \cite{lawrence1997face}. However, it was arguably not until 2012, when the GPU-based AlexNet CNN outperformed all competition on the ImageNet Large Scale Visual Recognition Challenge \cite{krizhevsky2012imagenet}, that CNNs and other neural networks truly caught the public eye. Industrial work and academic research on deep learning has since soared, and current state-of-the-art deep learning architectures are significantly more powerful and more complex than AlexNet. Yet, convolutional layers remain important components.

In this introduction, we focus on data that can be represented by finitely supported functions
\begin{equation}\label{eq:CNNdata}
f : \mathbb{Z}^2 \to \mathbb{R}^m.
\end{equation}
Digital images, for example, are of this form since each pixel $x \in \mathbb{Z}^2$ is associated with a color array $f(x) \in \mathbb{R}^m$, and finite support is analogous to finite image resolution. Note that $m = 1$ corresponds to grayscale images and $m = 3$ to RGB images, but we allow any number of channels $m$.  In general, any data represented by a finite 2D ($m=1$) or 3D  ($m>1$) array with real-valued entries is of the form \eqref{eq:CNNdata}.

\emph{Convolutional layers} act on data points \eqref{eq:CNNdata} by\footnote{The name \emph{convolutional layer} is used even \mbox{though \eqref{eq:convlayer_sec2} more closely resembles a cross-correlation}. It can be expressed as a convolution if we replace the kernel with its involution $\kappa^*(y) = \kappa(-y)$.}
\begin{equation}\label{eq:Z2conv}
[\kappa \star f](x) = \sum_{y \in \mathbb{Z}^2} \kappa(y-x) f(y),
\end{equation}
given a matrix-valued kernel $\kappa : \mathbb{Z}^2 \to \mathrm{Hom}(\mathbb{R}^m,\mathbb{R}^n)$ for some $n \in \mathbb{N}$. The kernel is also finitely supported in practice, so the maps $\kappa \star f : \mathbb{Z}^2 \to \mathbb{R}^n$ are themselves data points \eqref{eq:CNNdata} with $n$ channels. Broadly speaking, CNNs consist of convolutional layers \eqref{eq:Z2conv} combined with other transformations, such as non-linear activation functions and batch normalization layers. We are mainly interested in convolutional layers, so we do not go into detail about non-linear activation functions or other types of layers. For more extensive descriptions of CNNs, see \cite{albawi2017understanding,Goodfellow-et-al-2016,wu2017introduction}.

In image classification tasks, for instance, CNNs categorize digital images into a pre\-defined number $k$ of distinct classes, based on what the images depict. The CNN maps each digital image $f : \mathbb{Z}^2 \to \mathbb{R}^m$ to a probability vector in $\mathbb{R}^k$ estimating the probability that $f$ belongs to any given class. During training, this probability vector is compared to the correct answer (which is known) and the discrepancy is computed using a \emph{loss} norm or distance function. A gradient descent-based algorithm minimizes the loss function, thereby \emph{learning} the kernel matrix elements and any other trainable network parameters. The result of this training procedure is a CNN that accurately classifies images in the training data set. Finally, the predictive power of the CNN is evaluated by using it to classify images from a test data set; images that were not used during training and which the CNN has not encountered before.

CNNs perform very well on image classification and similar machine learning tasks, and are important parts of many state-of-the-art network architectures on such tasks  \cite{brock2021high,jia2021scaling,tan2021efficientnetv2,wu2021cvt}. One reason for their success is \emph{translation equivariance}: Convolutional layers \eqref{eq:Z2conv} commute with the translation operator in the image plane,
\begin{equation}
L_x : \mathbb{Z}^2 \to \mathbb{Z}^2, \qquad L_x(y) = y+x, \qquad x \in \mathbb{Z}^2.
\end{equation}
Translation equivariance makes CNNs agnostic to the specific locations of individual pixels, while still taking into account the relative positions of different pixels; images are more easily classified based on relevant features of their subjects, and not based on technical artifacts such as specific pixel coordinates. This observation motivates  the introduction of more general convolutional layers that act equivariantly on data points $f : \mathcal{M} \to V$, where the domain  $\mathcal{M}$ is homogeneous with respect to a locally compact group $G$ \cite{cohen2016group,cohen2018general}. Given finite-dimensional vector spaces $V,W$, convolutional layers are defined as certain vector-valued integrals\footnote{For a summary on vector-valued integration on locally compact groups, see \cite[Appendix 4]{folland2016course}.}
\begin{equation}\label{eq:convlayer_sec2}
\kappa \star f : \mathcal{M} \to W, \qquad (\kappa \star f)(g) = \int_G \kappa(g^{-1}g') f(g') \ \dd g',
\end{equation}
with operator-valued kernels $\kappa : G \to \mathrm{Hom}(V,W)$.

\begin{remark}
In \eqref{eq:convlayer_sec2}, we integrate with respect to a Haar measure on the unimodular Lie group $G$.  
\end{remark}

Broadly speaking, $G$-equivariant convolutional neural networks (GCNNs) consist of sequences of convolutional layers \eqref{eq:convlayer_sec2} mixed with non-linear activation functions, and possibly other layers that are equivariant with respect to the global symmetry. This characterization is intentionally vague as we want to avoid making unnecessarily restrictive assumptions on the layers. For this reason, we will not study GCNNs from a holistic perspective, as a sequence of multiple layers, but instead focus on individual layers. We give a formal definition of abstract, $G$-equivariant layers in Definition \ref{def:Glayer}, before defining the more specific convolutional layers in Definition \ref{def:convlayer}.

\subsection{Gauge theory and the equivariant framework}\label{subsec:ENN} 

Before going into detail about GCNNs in Section \ref{sec:GCNN}, let us describe a mathematical framework for equivariant neural networks. The framework is based on gauge theoretic concepts but is equally suitable for GCNNs. Both gauge equivariant neural networks and GCNNs will thus be described by this framework, their main difference being the specific equivariance properties imposed on layers.

\begin{remark}
This framework is already being used in GCNNs and gauge equivariant neural networks separately \cite{cheng2019covariance,cohen2018general}. We are simply presenting the unified theory that includes both types of equivariance as separate cases.
\end{remark}

Gauge theory originated in physics as a way to model local symmetry. In quantum electrodynamics (QED), for example, the electron wave function can be locally phase shifted, $\psi \mapsto e^{i\alpha} \psi$, with no physically observable consequence, and so QED is said to possess a $U(1)$ gauge symmetry. Mathematicians have later adopted gauge theory in order to study other types of local symmetries. The introduction of gauge equivariant deep learning models has been suggested by deep learning practitioners and physicists alike. For example, \cite{cheng2019covariance} investigates the structure of gauge equivariant layers used for vector fields, tensor fields, and more general fields. Physicists have introduced gauge equivariant neural networks for applications in, e.g., lattice gauge theory  \cite{boyda2021finding,favoni2020lattice,luo2021gauge}.

We assume some familiarity with fiber bundles,\footnote{Introductions to fiber bundles can be found in \cite{kolar2013natural,lee2013smooth,nakahara2003geometry}.} but we still present a few relevant definitions and examples.

\begin{definition}\label{def:principal_bundle}
Let $K$ be a Lie group. A smooth fiber bundle $\pi : P \to \mathcal{M}$ is called a \emph{principal $K$-bundle} with \emph{structure group} $K$ if there is  a free, smooth right $K$-action
\begin{equation}
P \times K \to P, \qquad (p,k) \mapsto p \triangleleft k,
\end{equation}
with the following properties for each $x \in \mathcal{M}$.
\begin{enumerate}
\item[(i)]  Let $P_x = \pi^{-1}(\{x\})$ be the fiber at $x$. Then
\begin{equation}
p \in P_x  , \ k \in K \quad \Rightarrow \quad p \triangleleft k \in P_x.
\end{equation}
That is, the $K$-action preserves fibers.
\item[(ii)] For each $p \in P_x$, the mapping $k \mapsto p \triangleleft k$ is a diffeomorphism $K \to P_x$.
\end{enumerate}
\end{definition}

Principal bundles are natural tools for understanding local symmetries, i.e., gauge degrees of freedom. In theoretical physics, gauge degrees of freedom are redundancies in the mathematical theory with no physical relevance. This is both a blessing and a curse: Solving the Yang-Mills equations of motion as an initial value problem, for example, is an underdetermined problem that cannot be solved without taking the gauge degrees of freedom into account; without choosing a  \emph{gauge} \cite{Muller2019YangMills}. This is similar to our example in the introduction, that computations involving vector fields may require a choice of basis in each tangent space, even if this choice is irrelevant for the underlying application. On the other hand, problems may also become easier to solve by choosing a gauge with some finesse.\footnote{See, for example, the temporal gauge in lattice gauge theory \cite[\S 3.3.2]{gattringer2009quantum}.}

\begin{definition}
Let $\pi : P \to \mathcal{M}$ be a principal $K$-bundle and assume $U \subseteq \mathcal{M}$ is open.
\begin{enumerate}
\item[(i)] A \emph{gauge} is a local section $\omega : U \to P$.

\item[(ii)] A \emph{gauge transformation} is an automorphism $\chi : P \to P$ that is equivariant,
\begin{equation}
\chi( p \triangleleft k) = \chi(p) \triangleleft k, \qquad p \in P, k \in K,
\end{equation}
and which preserves fibers: $\pi \circ \chi = \pi$.
\end{enumerate}
\end{definition}

We will go into more detail about the vector field example in Example \ref{ex:vector_field}. However, we first need to define associated bundles. To this end, let $\pi : P \to \mathcal{M}$ be a principal $K$-bundle and let $\rho : K \to GL(V_\rho)$ be a finite-dimensional representation. Define an equivalence relation $\sim$ on $P \times V_\rho$ by
\begin{equation}
(p,v) \sim (p \triangleleft k , \rho(k)^{-1} v), \qquad p \in P, v \in V_\rho, k \in K.
\end{equation}
Let $P \times_\rho V_\rho = (P \times V_\rho) / \sim$ denote the quotient space, whose elements are equivalence classes
\begin{equation}
[p,v] = [p \triangleleft k , \rho(k)^{-1}v], \qquad p \in P, v \in V_\rho, k \in K,
\end{equation}
and consider the projection $\pi_\rho : P \times_\rho V_\rho \to \mathcal{M}$ defined by $\pi_\rho([p,v]) = \pi(p)$. Observe that each fiber $\pi_\rho^{-1}(\{x\})$ has a natural vector space structure such that the mapping
\begin{equation}
V_\rho \to \pi_\rho^{-1}(\{x\}), \qquad v \mapsto [p,v],
\end{equation}
is a linear isomorphism for each fixed $p \in P_x$.

\begin{lemma}[{\cite[\S 10.7]{kolar2013natural}}]
The \emph{associated bundle} $\pi_\rho : P \times_\rho V_\rho \to \mathcal{M}$ is a smooth vector bundle.
\end{lemma}

\begin{example}\label{ex:vector_field}
Let $d = \dim \mathcal{M}$ and consider a coordinate chart $\left(u^1,\ldots,u^d\right) : U \to \mathbb{R}^d$, for $U \subseteq \mathcal{M}$. Recall that for each $x \in U$, the induced coordinate basis in $T_x\mathcal{M}$,
\begin{equation}\label{eq:local_frame}
\omega(x) := \left(  \frac{\partial}{\partial u^1}\Big|_x,\ldots,\frac{\partial}{\partial u^d}\Big|_x \right),
\end{equation}
lets us express tangent vectors $X_x \in T_x\mathcal{M}$ in components $X_x^1,\ldots,X_x^d \in \mathbb{R}$. Moreover, \eqref{eq:local_frame} defines a local frame $\omega : U \to F\mathcal{M}$ that sends each point $x \in U$ to its coordinate basis in $T_x\mathcal{M}$. Local frames are sections of the frame bundle $F\mathcal{M}$, which is a principal $GL(d,\mathbb{R})$-bundle, hence local frames are examples of gauges.

Fix $x \in U$ and expand $X_x \in T_x\mathcal{M}$ in the coordinate basis:
\begin{equation}\label{eq:tangent_vector}
X_x = \sum_{i=1}^d X_x^i \frac{\partial}{\partial u^i} \Big|_x.
\end{equation}
By decomposing \eqref{eq:tangent_vector} into components and basis vectors, we can view the tangent vector $X_x$ as one giant tuple
\begin{equation}
X_x = \left( \frac{\partial}{\partial u^1}\Big|_x,\ldots,\frac{\partial}{\partial u^d}\Big|_x , X_x^1 , \ldots, X_x^d\right) = (\omega(x), X(x)) \in F\mathcal{M} \times \mathbb{R}^d.
\end{equation}

Another choice of coordinate chart produces another local frame $\omega' : U' \to F\mathcal{M}$ and another decomposition \eqref{eq:tangent_vector}, assuming that $x \in U'$. These decompositions are related by a change of basis
\begin{equation}\label{eq:jacobian}
\left( \omega'(x), X'(x)\right) = (\omega(x) B(x) , B(x)^{-1}X(x)),
\end{equation}
for some $B(x) \in GL(d,\mathbb{R})$. Now observe that \eqref{eq:jacobian} is of the form
\begin{equation}
(p' , v') = (p \triangleleft k , \rho(k)^{-1}v),
\end{equation}
where $p,p' \in F\mathcal{M}$, $v,v' \in \mathbb{R}^d$, $k \in GL(d,\mathbb{R})$, the right-action $\triangleleft$ is right-multiplication, and $(\rho,\mathbb{R}^d)$ is the standard representation $\rho(k) = k$ of $GL(d,\mathbb{R})$. The basis-dependent description \eqref{eq:tangent_vector} of $X_x$ thus resemble the pairs $(p,v)$ in the construction of associated bundles. Passing to the quotient $(F\mathcal{M} \times_\rho \mathbb{R}^d)/\sim$ instead gives a basis-\emph{in}dependent description of $X_x$, since it identifies all possible decompositions \eqref{eq:tangent_vector} in all possible bases. That is, the tangent bundle is isomorphic to  $F\mathcal{M} \times_\rho \mathbb{R}^d$.
\end{example}

Equivariant neural networks use the language of principal and associated bundles. In the remainder of this subsection, let $E_\rho = P \times_\rho V_\rho$ and $E_\sigma = P \times_\sigma V_\sigma$ be associated bundles, given a principal bundle $\pi : P \to \mathcal{M}$ over a smooth manifold $\mathcal{M}$. Further let $\Gamma_c(E_\rho)$ and $\Gamma_c(E_\sigma)$ be the vector spaces of compactly supported continuous sections of $E_\rho$ and $E_\sigma$, respectively.

\begin{definition}\label{def:datapoint}
A \emph{data point} is a section $s \in \Gamma_c(E_\rho)$.
\end{definition}

\begin{remark} Even though data is typically real-valued, we primarily consider complex representations $(\rho,V_\rho)$ so to simplify the mathematical theory. The harmonic analysis in Section \ref{subsec:RKHS} especially benefits from this choice.

Our decision to restrict attention to compactly supported sections was also made for mathematical reasons: $G$-equivariant layers are defined in Section \ref{subsec:group} in terms of an induced representation, which lives on the completion of $\Gamma_c(E_\rho)$ with respect to a certain inner product. This is not a serious restriction from an application viewpoint.
\end{remark}

\begin{definition}
A \emph{feature map} is a compactly supported continuous map $f : P \to V_\rho$ that satisfies the transformation property
\begin{equation}\label{eq:feature_map}
f(p \triangleleft k) = \rho(k)^{-1}f(p),
\end{equation}
for all $p \in P$, $k \in K$.
The vector space of such feature maps is denoted $C_c(P;\rho)$
\end{definition}

Data points and feature maps are, in a sense, dual to each other: Each data point in $ \Gamma_c(E_\rho)$ is of the form
\begin {equation}\label{eq:section_iso}
s_f(x) = [p,f(p)],
\end{equation}
for a feature map $f \in C_c(P;\rho)$, where $p \in P_x$ is any element of the fiber at $x \in \mathcal{M}$.  Note that \eqref{eq:section_iso} does not depend on the choice of $p$: Given another element $p' \in P_x$, there exists a unique $k \in K$ such that $p' = p \triangleleft k$ and
\begin{equation}
[p',f(p')] = [p \triangleleft k , f(p\triangleleft k)] = [p \triangleleft k , \rho(k^{-1})f(p)] = [p,f(p)].
\end{equation}
That is, the equivalence class $[p,f(p)]$ only depends on the basepoint $x$.

\begin{lemma}[{\cite[\S 10.12]{kolar2013natural}}]\label{lemma:section_iso}
The linear map $C_c(P;\rho) \to \Gamma_c(E_\rho)$, $f \mapsto s_f$ is a vector space isomorphism.
\end{lemma}

We are almost ready to define general and gauge equivariant layers. Before doing so, however, we must say how gauge transformations $\chi : P \to P$ act on data points. Let $\theta_\chi : P \to K$ be the uniquely defined map satisfying $\chi(p) = p \triangleleft \theta_\chi(p)$ for all $p \in P$, and define the following action on the associated bundle $E_\rho$:
\begin{equation}
\chi \cdot [p,v] = [\chi(p),v] = [p, \rho(\theta_\chi(p))v], \qquad [p,v] \in E_\rho.
\end{equation}
The corresponding action on data points is given by
\begin{equation}
(\chi  \cdot s_f)(p) = [p , \rho(\theta_\chi(p)) f(p)] = s_{\rho(\theta_\chi)f}(x), \qquad s_f \in \Gamma_c(E_\rho).
\end{equation}
We distinguish between general layers and more specific gauge equivariant layers, as $G$-equivariant layers in GCNNs will only be a special case of the former.
\begin{definition}\label{def:layer}
Let $E_\rho = P \times_\rho V_\rho$ and $E_\sigma = P \times_\sigma V_\sigma$ be associated bundles.
\begin{enumerate}
\item[(i)]  A \emph{(linear) layer} is a linear map $\Phi : \Gamma_c(E_\rho) \to \Gamma_c(E_\sigma)$.
\item[(ii)] A layer $\Phi$ is \emph{gauge equivariant} if, for all gauge transformations $\chi : P \to P$,
\begin{equation}\label{eq:gaugeequivariance}
\Phi \circ \chi = \chi \circ \Phi.
\end{equation}
\end{enumerate}
\end{definition}

In equivariant neural networks, data points are sent through a sequence of layers, which are mixed with non-linear activation functions. Again, we focus on individual layers in this paper, and leave the analysis of equivariant activation functions and multi-layer networks for an upcoming paper \cite{gerken2021geometric}. The fiber bundle-theoretic concepts discussed in this part describe two kinds of equivariant neural networks:
\begin{enumerate}
\item[(i)] Gauge equivariant neural networks, which respect local gauge symmetry and whose layers are gauge equivariant.
\item[(ii)] GCNNs, which respect global translation symmetry in homogeneous $G$-spaces $\mathcal{M}$, and whose layers are $G$-equivariant (Definition \ref{def:Glayer}).
\end{enumerate}

\begin{remark}
Our definition of layers is almost identical to the linear maps in \cite{cheng2019covariance}. The difference is that we focus on compactly supported sections, whereas \cite{cheng2019covariance} use sections that are supported on a single coordinate chart. Also, \cite{cheng2019covariance} investigates the structure of their linear maps under additional assumptions of so-called locality, covariance, and weight-sharing. Covariance is analogous to gauge equivariance in our setting.
\end{remark}

A consequence of Lemma \ref{lemma:section_iso} is that (gauge equivariant) layers $\Phi : \Gamma_c(E_\rho) \to \Gamma_c(E_\sigma)$ induce a unique linear map $\phi : C_c(P;\rho) \to C_c(P;\sigma)$ such that $\Phi s_f = s_{\phi f}$. Writing data points as $s_f = [\cdot,f]$ allows us to also express this relation as $\Phi [\cdot , f] = [\cdot , \phi f]$. We think of $\Phi$ and $\phi$ as two sides of the same coin, and use the name (gauge equivariant) layer for both maps.
\[
\begin{tikzcd}
\Gamma_c(E_\rho) \arrow[dd, shift left=.75ex] \arrow{rr}{\Phi} && \Gamma_c(E_\sigma) \arrow[dd, shift left=.75ex]\\
&\\
C_c(P;\rho) \arrow[uu, shift left=.75ex] \arrow[dashed]{rr}{\phi} && C_c(P;\sigma)\arrow[uu, shift left=.75ex] 
\end{tikzcd}
\]

\begin{example}
Let $T : V_\rho \to V_\sigma$ be a linear transformation and consider the layer
\begin{equation}
\phi : C_c(P;\rho) \to C_c(P;\sigma), \qquad (\phi f)(p) = T\big(f(p)\big),
\end{equation}
for $p \in P$, $f \in C_c(P;\rho) $. Since $f$ and $\phi f$ are feature maps and thereby satisfy \eqref{eq:feature_map}, the linear transformation $T$ must satisfy
\begin{alignat}{1}
\sigma(k)T\big(f(p)\big) = T\big(f(p \triangleleft k^{-1})\big) = T\big( \rho(k) f(p)\big),
\end{alignat}
for all $k \in K, p \in P, f \in C_c(P;\rho)$. This can be seen to imply that $\sigma \circ T = T \circ \rho$, so $T$ intertwines the representations $\rho$ and $\sigma$. Another way to arrive at this conclusion is to analyze when the corresponding layer
\begin{equation}\label{eq:evaluatemorphism}
\Phi : \Gamma_c(E_\rho) \to \Gamma_c(E_\sigma), \qquad \Phi s_f = [\cdot, \phi f],
\end{equation}
is well-defined.

Now consider a gauge transformation $\chi : P \to P$ and its induced map $\theta_\chi : P \to K$. Because $T$ is an intertwiner,
\begin{equation}\label{eq:evaluatemorphism2}
(\Phi \circ \chi)s_f = s_{\phi \rho(\theta_\chi)f} = s_{\sigma(\theta_\chi)\phi f} = (\chi \circ \Phi)s_f,
\end{equation}
hence the layer $\Phi$ is automatically gauge equivariant.
\end{example}

As this example illustrates, gauge equivariance is tightly connected to intertwining properties of $\phi$. Rearranging \eqref{eq:evaluatemorphism2} gives the following result.

\begin{lemma}
A general layer $\Phi : \Gamma_c(E_\rho) \to \Gamma_c(E_\sigma)$ is gauge equivariant iff
\begin{equation}
\phi \circ \rho(\theta_\chi)f = \sigma(\theta_\chi) \circ \phi f,
\end{equation}
for all gauge transformations $\chi : P \to P$ and all feature maps $f \in C_c(P;\rho)$.
\end{lemma}

This concludes our discussion of gauge theory and of equivariant neural networks. The framework for the latter is evidently very general, consisting of layers and non-linear activation functions between data points. There are advantages of working at this level of generality: Ordinary (non-equivariant) neural networks have a multitude of different types of layers, many of them linear. Equivariant analogues of such layers are likely to satisfy either Definition \hyperref[def:layer]{\ref*{def:layer}(ii)} or Definition \ref{def:Glayer}, depending on the relevant type of equivariance. Any result that can be proven using this general framework, will thus be true for many different instances  of equivariant neural networks. One example is Theorem \ref{thm:main} below, that characterizes the structure of abstract $G$-equivariant layers in any GCNN.

\section{\texorpdfstring{$G$}{G}-equivariant convolutional neural networks}\label{sec:GCNN}

Recall that GCNNs generalize ordinary CNNs to data points $f : \mathcal{M} \to V$ defined on homogeneous $G$-spaces $\mathcal{M}$. Let us give a brief recap on homogeneous spaces and global symmetry, before moving on to discuss homogeneous vector bundles, sections, and induced representations. We will demonstrate that GCNNs and $G$-equivariant layers (originally defined in \cite{cohen2018general}) are most naturally understood from the perspective of homogeneous vector bundles. We then use reproducing kernel Hilbert spaces and bandwidth to understand which $G$-equivariant layers are expressible as convolutional layers.

\subsection{Homogeneous spaces}

\begin{definition}
Let $G$ be a Lie group. A smooth manifold $\mathcal{M}$ is called a \emph{homogeneous $G$-space} if there exists a smooth, transitive left $G$-action
\begin{equation}\label{eq:G_action}
G \times \mathcal{M} \to \mathcal{M}, \qquad (g,x) \mapsto g \cdot x.
\end{equation}
\end{definition}

Since the action \eqref{eq:G_action} is transitive, we may choose an arbitrary basepoint $x_0 \in \mathcal{M}$ and express any other point $x \in \mathcal{M}$ as $x = g \cdot x_0$ for some $g \in G$. This group element is typically not unique, but observe that
\begin{equation}
g \cdot x_0 = g' \cdot x_0 \quad \iff \quad g^{-1}g' \in H_{x_0},
\end{equation}
where $H_{x_0} = \left\{ g \in G \ \middle| \ g \cdot x_0 = x_0\right\}$ is the isotropy group of $x_0$. In other words, there is a one-to-one correspondence between points $x \in \mathcal{M}$ and left cosets $gH_{x_0} \in G/H_{x_0}$.

\begin{proposition}[{\cite[Theorem 21.18]{lee2013smooth}}]\label{prop:homspace}
Let $\mathcal{M}$ be a homogeneous $G$-space and choose a basepoint $x_0 \in \mathcal{M}$. The isotropy group $H_{x_0}$ is a closed subgroup of $G$, and the map
\begin{equation}\label{eq:equiv_diffeo}
F_{x_0} : G/H_{x_0} \to \mathcal{M} , \qquad g H_{x_0} \mapsto g \cdot x_0,
\end{equation}
is an equivariant diffeomorphism.
\end{proposition}

Homogeneous spaces are globally symmetric in the sense that any point $x_0 \in \mathcal{M}$ may be chosen as basepoint. Given another choice of basepoint $x_0' \in \mathcal{M}$, the spaces $G/H_{x_0'} \simeq G/H_{x_0}$ are diffeomorphically related by a translation in $G$ - more precisely, by the composition $F_{x_0}^{-1} \circ F_{x_0'}$. Euclidean space $\mathcal{M} = \mathbb{R}^d$, for example, possesses a global translation symmetry, allowing any point to be considered as origin. Similarly, the rotationally symmetric sphere $\mathcal{M} = S^2$ does not have a unique north pole.

We end this part with the following proposition, which is instrumental in relating homogeneous vector bundles to the equivariance framework in Section \ref{subsec:ENN}.

\begin{proposition}[{\cite[\S 7.5]{steenrod1960topology}}]
Let $G$ be a Lie group and let $H \leq G$ be a closed subgroup. Then the quotient map
\begin{equation}
q  : G \to G/H, \qquad g \mapsto gH,
\end{equation}
defines a smooth principal $H$-bundle over the homogeneous $G$-space $\mathcal{M} = G/H$.
\end{proposition}

\subsection{Homogeneous vector bundles}\label{subsec:hombundle} Vector bundles may inherit global symmetry from a homogeneous base space; the transitive action $(g,x) \mapsto g\cdot x$ may induce linear maps $E_x \mapsto E_{gx}$ between fibers. Such bundles are naturally called \emph{homogeneous} and, because this symmetry is also encoded in its sections (data points), we will show that homogeneous vector bundles is the natural setting for studying GCNNs.

From this point on, we restrict attention to homogeneous spaces $\mathcal{M} = G/K$ where $G$ is a unimodular Lie group and $K \leq G$ is a compact subgroup. Elements of the homogeneous space is interchangably denoted as $x \in \mathcal{M}$ or $gK \in G/K$.

\begin{remark}
Examples of unimodular Lie groups include all finite, discrete, compact, or abelian Lie groups, the Euclidean groups, and many others. See \cite{folland2016course,fuhr2005abstract} for details.
\end{remark}

\begin{definition}[{\cite[5.2.1]{wallach2018harmonic}}]\label{def:hombundle}
Let $\mathcal{M}$ be a homogeneous $G$-space and let $\pi : E \to \mathcal{M}$ be a smooth vector bundle with fibers $E_x$. We say that $E$ is \emph{homogeneous} if there is a smooth left $G$-action $G \times E \to E$ satisfying
\begin{equation}\label{eq:bundleaction}
g \cdot E_x = E_{gx},
\end{equation}
and such that the induced map $L_{g,x} : E_x \to E_{gx}$ is linear, for all $g \in G, x \in \mathcal{M}$.
\end{definition}

\begin{example}\label{ex:FM}
The frame bundle $F\mathcal{M}$ is a homogeneous vector bundle whenever $\mathcal{M}$ is a homogeneous space, and the same is true of any associated bundle $F\mathcal{M} \times_\rho V_\rho$. In particular, the tangent bundle $T\mathcal{M}$ is a homogeneous vector bundle.
\end{example}

\begin{example}
If $(\rho,V_\rho)$ is a finite-dimensional $K$-representation, then the associated bundle $E_\rho = G \times_\rho V_\rho$ is a homogeneous vector bundle with respect to the left action
\begin{equation}
g \cdot [g', v] = [gg', v].
\end{equation}
\end{example}

All homogeneous vector bundles $E$ are of the form $G \times_\rho V_\rho$, up to isomorphism. To understand why, consider the fiber $E_K = E_{eK}$ and observe that the restriction of \eqref{eq:bundleaction} to $E_K$ and elements $k \in K$ yields invertible linear maps
\begin{equation}\label{eq:homreps}
L_{k} : E_K \to E_K.
\end{equation}
The defining properties of group actions ensure that $\rho(k) = L_{k}$ is a finite-dimensional $K$-representation on $E_K$. Moreover, because the linear maps $L_{g,x}$ are isomorphisms, any element $v'$ of any fiber $E_x$ can be obtained as the image $v' = L_{g,K}(v) =: L_g(v)$ for some choices of $g \in q^{-1}(\{x\})$ and $v \in E_K$. The mapping
\begin{equation}
\xi : G \times E_K \to E, \qquad (g,v) \mapsto L_g(v),
\end{equation}
is thus surjective. It is not injective, though, since the relation
\begin{equation}
L_{g} = L_{g} \circ L_{k} \circ L_{k^{-1}} = L_{gk} \circ \rho(k^{-1}),
\end{equation}
implies that $\xi(g,v) = L_g(v) = L_{gk}( \rho(k^{-1})v) = \xi(gk, \rho(k^{-1})v)$ for $k \in K$. However, the same argument shows that $\xi$ is made injective by passing to the quotient $G \times_\rho E_K$.

\begin{lemma}[{\cite[5.2.3]{wallach2018harmonic}}]\label{lemma:hombundle}
The map
\begin{equation}
 G \times_\rho E_K \to E, \qquad [g,v] \mapsto L_{g}(v),
\end{equation}
is an isomorphism of homogeneous vector bundles.
\end{lemma}

We now have two perspectives on bundles $G \times_\rho V_\rho$: As bundles associated to the principal bundle $P = G$, and as homogeneous vector bundles (up to isomorphism). The former perspective offers a connection to the framework in Section \ref{subsec:ENN}, whereas the latter motivates the definition of $G$-equivariant layers in Section \ref{subsec:group} below.

\subsection{Induced representations}\label{subsec:indrep} Let us show the relationship between homogeneous vector bundles and induced representations, which will be an essential ingredient in the definition of $G$-equivariant layers. To this end, let $(\rho, V_\rho)$ be a finite-dimensional unitary $K$-representation and consider the homogeneous vector bundle $E_\rho = G \times_\rho V_\rho$.

We will need inner products on $\Gamma_c(E_\rho)$ and $C_c(G;\rho)$, the former of which is defined using the following unitary structure:

\begin{lemma}[{\cite[5.2.7]{wallach2018harmonic}}]
The unitary structure
\begin{equation}\label{eq:unitarystructure}
\langle [g,v] , [g,w]\rangle_{gK} := \langle v, w\rangle_\rho,
\end{equation}
defines a complete inner product on each fiber $E_{gK}$, making $E_\rho$ into a Hilbert bundle with $L_{g,x}$ unitary. This unitary structure is unique in that, if we identify $V_\rho$ with $E_{K}$ in the canonical manner, then the inner product on $V_\rho$ so induced agrees with $\langle \ , \ \rangle_\rho$.
\end{lemma}

We also need the following measure on $G/K$:

\begin{theorem}[Quotient Integral Formula {\cite[\S 1.5]{deitmar2014principles}}]
There is a unique $G$-invariant, nonzero Radon measure $\dd x$ on $G/K$ such that the following \emph{quotient integral formula} holds for every $f \in C_c(G)$:
\begin{equation}\label{eq:qif}
\int_G f(g) \ \dd g = \int_{G/K} \int_K f(xk) \ \dd k \ \dd x.
\end{equation}
\end{theorem}

Using these two ingredients, we make $\Gamma_c(E_\rho)$ into a pre-Hilbert space with respect to the inner product
\begin{equation}\label{eq:innerprod1}
\langle s,s'\rangle_{L^2(E_\rho)} := \int_{G/K} \langle s(x), s'(x)\rangle_x \ \dd x, \qquad s,s' \in \Gamma_c(E_\rho),
\end{equation}
and we denote its completion $L^2(E_\rho)$. Similarly, $C_c(G;\rho)$ is a pre-Hilbert space with respect to the inner product
\begin{equation}\label{eq:innerprod2}
\langle f, f' \rangle_{L^2(G;\rho)} = \int_G \langle f(g), f'(g) \rangle_\rho \ \dd g, \qquad f,f' \in C_c(G;\rho),
\end{equation}
the completion of which is denoted $L^2(G;\rho)$.

\begin{definition}\label{def:indreps}
The $G$-representations
\begin{alignat}{3}
&\mathrm{ind}_K^G\rho(g) : L^2(E_\rho) \to L^2(E_\rho), \qquad &&(\mathrm{ind}_K^G\rho(g) s)(x) &&= g\cdot s(g^{-1}x),\\
&\mathrm{Ind}_K^G\rho(g) : L^2(G;\rho) \to L^2(G;\rho), \qquad &&(\mathrm{Ind}_K^G\rho(g) f)(g') &&= f(g^{-1}g').
\end{alignat}
are called \emph{induced representations}, or \emph{representations induced by $\rho$}.
\end{definition}
Both $\mathrm{ind}_K^G \rho$ and $\mathrm{Ind}_K^G \rho$ are unitary \cite[5.3.2]{wallach2018harmonic} and may be identified:

\begin{lemma}
The induced representations $\mathrm{ind}_K^G(\rho)$, $\mathrm{Ind}_K^G(\rho)$ are unitarily equivalent.
\end{lemma}
\begin{proof}
This is \cite[5.3.4]{wallach2018harmonic}, but let us write down a proof for clarity. First observe that the isomorphism $C_c(G;\rho) \to \Gamma_c(E_\rho)$, $f \mapsto s_f$ is unitary, which follows by combining the quotient integral formula \eqref{eq:qif}, the unitarity of $\rho$, and the compactness of $K$: For all $f,f' \in C_c(G;\rho)$, the map $g \mapsto \langle f(g) , f'(g)\rangle_\rho$ lies in $C_c(G)$ and so
\begin{equation}
\begin{aligned}
\langle f, f'\rangle_{L^2(G;\rho)} &= \int_{G/K} \int_K \langle f(xk) , f'(xk)\rangle_\rho \ \dd k \ \dd x\\
&= \int_{G/K} \langle f(x) , f'(x)\rangle_\rho \ \dd x\\
&= \int_{G/K} \langle [x, f(x)], [x, f'(x)]\rangle_x \ \dd x = \langle s_f, s_{f'}\rangle_{L^2(E_\rho)}.
\end{aligned}
\end{equation}
The same map $f \mapsto s_f$ satisfies
\begin{equation}
\left(\mathrm{ind}_K^G(\rho) s_f\right)(x) = g \cdot s_f(g^{-1}x) = [x, f(g^{-1}x)] = s_{\mathrm{Ind}_K^G(\rho) f}(x),
\end{equation}
so it extends to a unitary isomorphism $L^2(G;\rho) \to L^2(E_\rho)$ intertwining the induced representations.
\end{proof}

To gain a better understanding of the induced representations, consider the Bochner space $L^2(G,V)$, the space of square-integrable functions $f : G \to V$ that take values in a finite-dimensional Hilbert space $V$. It is itself a Hilbert space with inner product
\begin{equation}
\langle f, f' \rangle_{L^2(G,V)} = \int_G \langle f(g), f'(g)\rangle_V \ \dd g'.
\end{equation}
The induced representation $(\mathrm{Ind}_K^G \rho, L^2(G;\rho))$ is nothing but the restriction of the left regular representation $\Lambda$ on $L^2(G,V_\rho)$ to a closed, invariant subspace. Furthermore, $\Lambda$ is intimately related to the left regular representation $\lambda$ on $L^2(G)$, as the following lemma shows. The proof of this lemma is a short calculation. 

\begin{lemma}\label{lemma:L2equiv}
Let $V$ be a finite-dimensional Hilbert space and equip $L^2(G) \otimes V$ with the tensor product inner product. Then the natural unitary isomorphism
\begin{equation}\label{eq:unitary_equiv}
\begin{aligned}
A : L^2(G) \otimes V &\to L^2(G,V)\\ f \otimes v &\mapsto fv
\end{aligned}
\end{equation}
intertwines $\lambda \otimes \Id_V$ with $\Lambda$.
\end{lemma}

This lemma also shows that, if we choose an orthonormal basis $e_1,\ldots,e_{\dim V} \in V$, elements of $L^2(G,V)$ are simply linear combinations $f = \sum_i f^i e_i$ with component functions $f^i \in L^2(G)$. We use this fact in some calculations of vector-valued integrals, and the component functions will also be important in Section \ref{subsec:RKHS}.

\subsection{\texorpdfstring{$G$}{G}-equivariant and convolutional layers}\label{subsec:group} 

Given a homogeneous $G$-space $\mathcal{M}$, we observed that vector bundles $\pi : E \to \mathcal{M}$ may inherit the  global symmetry of $\mathcal{M}$. We took a closer look at such homogeneous vector bundles and found that they are isomorphic to associated bundles $G \times_\rho V_\rho$, and therefore fit within the equivariance framework of Section \ref{subsec:ENN} . We also saw how the global symmetry of $\mathcal{M}$ is encoded in data points and feature maps via induced representations, and we want $G$-equivariant layers to preserve this global symmetry.

Consider homogeneous vector bundles $E_\rho = G \times_\rho V_\rho$ and $E_\sigma = G \times_\sigma V_\sigma$, and recall Definition \ref{def:layer} of layers as general linear maps $\Phi : \Gamma_c(E_\rho) \to \Gamma_c(E_\sigma)$. We are mainly interested in \emph{bounded} layers from an application point of view, and we can make this restriction now that the domain and codomain are normed spaces. Furthermore, any bounded layer can be uniquely extended to a bounded linear map
\begin{equation}\label{eq:extension}
\Phi : L^2(E_\rho) \to L^2(E_\sigma),
\end{equation}
and we assume this extension has already been made.

\begin{definition}\label{def:Glayer}
A bounded linear map $\Phi : L^2(E_\rho) \to L^2(E_\sigma)$ is called a \emph{$G$-equivariant layer} if it intertwines the induced representations:
\begin{equation}\label{eq:Glayer}
\Phi \circ \mathrm{ind}_K^G \rho = \mathrm{ind}_K^G \sigma \circ \Phi.
\end{equation}
That is, $G$-equivariant layers are elements $\Phi \in \mathrm{Hom}_G(L^2(E_\rho),L^2(E_\sigma))$.
\end{definition}

\begin{remark}
We could also have defined $G$-equivariant layers as bounded linear maps $\phi : L^2(G;\rho) \to L^2(G;\sigma)$ that intertwine the induced representations:
\begin{equation}
\phi \circ \mathrm{Ind}_K^G \rho = \mathrm{Ind}_K^G \sigma \circ \Phi,
\end{equation}
i.e., elements $\phi \in \mathrm{Hom}_G(L^2(G;\rho),L^2(G,\sigma))$. These definitions are clearly equivalent.
\end{remark}

Apart from minor technical differences, Definition \ref{def:Glayer} coincides with the definition of equivariant maps in \cite{cohen2018general}. We have thus obtained GCNNs almost directly from the definition of homogeneous vector bundles and a desire for layers to respect the global symmetry. This shows that homogeneous vector bundles is the natural setting for GCNNs.

Let us now define convolutional layers.

\begin{definition}\label{def:convlayer}
A \emph{convolutional layer} $L^2(G;\rho) \to L^2(G;\sigma)$ is a bounded operator
\begin{equation}\label{eq:convlayer}
[\kappa \star f](g) = \int_G \kappa(g^{-1}g') f(g') \ \dd g', \qquad f \in L^2(G;\rho),
\end{equation}
with an operator-valued kernel $\kappa : G \to \mathrm{Hom}(V_\rho,V_\sigma)$.
\end{definition}

Of course, not any function $\kappa : G \to \mathrm{Hom}(V_\rho,V_\sigma)$ can be chosen as the kernel of a convolutional layer. The kernel must ensure both that \eqref{eq:convlayer} is bounded and that $\phi f \in L^2(G;\sigma)$ for each $f \in L^2(G;\rho)$. We give a sufficient condition for boundedness in Lemma \ref{lemma:conv_bounded} and the other requirement has been studied in detail in  \cite{cohen2018general,lang2020wigner}.

The next result is an almost immediate consequence of the Fubini-Tonelli theorem.

\begin{proposition}
The adjoint of \eqref{eq:convlayer} is the integral operator
\begin{equation}\label{eq:convadjoint}
[f * \kappa^*](g) = \int_G \kappa^*({g'}^{-1}g) f(g') \ \dd g', \qquad f \in L^2(G;\sigma),
\end{equation}
where $\kappa^*$ is the pointwise adjoint of $\kappa$. That is, $(\kappa \star \cdot)^* = \cdot * \kappa^*$.
\end{proposition}

One way to ensure that the operators \eqref{eq:convlayer}-\eqref{eq:convadjoint} are bounded, is to put a bound on the kernel matrix elements $\kappa_{ij} : G \to \mathbb{C}$ for any given choice of bases in $V_\rho$, $V_\sigma$.

\begin{lemma}\label{lemma:conv_bounded}
The operators \eqref{eq:convlayer} and \eqref{eq:convadjoint} are bounded if $\kappa_{ij} \in L^1(G)$ for all $i,j$.
\end{lemma}
\begin{proof}
We need only prove that \eqref{eq:convadjoint} is bounded, its adjoint \eqref{eq:convlayer} will be bounded as well. Choose bases $e_1,\ldots,e_{\dim V_\rho} \in V_\rho$ and $\tilde{e}_1,\ldots,\tilde{e}_{\dim V_\sigma} \in V_\sigma$ and observe that, because $L^2(G;\sigma) \subset L^2(G,V_\sigma)$, Lemma \ref{lemma:L2equiv} enables the decomposition of $f \in L^2(G;\sigma)$ into component functions $f^i \in L^2(G)$:
\begin{equation}
f = \sum_{i=1}^{\dim V_\sigma} f^i \tilde{e}_i.
\end{equation}
To be clear, the kernel $\kappa$ is similarly decomposed into matrix elements $\kappa_{ij} = \langle \tilde{e}_j , \kappa e_i\rangle_{\sigma}$ and we have $\kappa_{ji}^* = \overline{\kappa_{ij}}$. The integral \eqref{eq:convadjoint} now takes the form
\begin{equation}
[f * \kappa^*](g) = \sum_{j=1}^{\dim V_\rho} \left(\int_G  \sum_{i=1}^{\dim V_\sigma} \kappa_{ji}^*(g'^{-1}g)f^i(g) \ \dd g\right) e_j,
\end{equation}
so by Young's convolution inequality,
\begin{equation}
\begin{aligned}
\| f * \kappa^*\|_{L^2(G;\rho)}^2 &\leq \sum_{i,j} \int_G \left| \int_G \kappa_{ji}^*({g'}^{-1} g) f^i(g') \ \dd g'\right|^2 \dd g = \sum_{i,j} \| f^i * \kappa_{ji}^*\|_{2}^2\\
&\leq  \sum_{i,j} \| \kappa_{ij}\|_{1}^2 \|f^i \|_{2}^2 \leq M \sum_{i} \|f^j\|_{2}^2 = M \|f\|_{L^2(G;\sigma)}^2,
\end{aligned}
\end{equation}  
where $M = \sum_{i,j} \|\kappa_{ij}\|_{1}^2 < \infty$ if $\kappa_{ij} \in L^1(G)$ for all $i,j$.
\end{proof}

We are interested in convolutional layers partly because they are concrete examples of $G$-equivariant layers, which we show next.

\begin{proposition}\label{prop:convlayers}
Convolutional layers are $G$-equivariant layers.
\end{proposition}
\begin{proof}
Convolutional layers $\kappa\star \cdot : L^2(G;\rho) \to L^2(G;\sigma)$ are bounded linear operators by definition, so the only thing we need to prove is that $\kappa \star \cdot$ intertwines the induced representations. This follows immediately from left-invariance of the Haar measure: For each $f \in L^2(G;\rho)$ and all $g,h \in G$,
\begin{equation}
\begin{aligned}
\big[\kappa \star  \big(\mathrm{Ind}_K^G \rho(g)f\big)\big](h) &= \int_G \kappa(h^{-1}g') f(g^{-1}g') \ \dd g' \qquad (g' \mapsto gg')\\
&= \int_G \kappa\big((g^{-1}h)^{-1} g'\big) f(g') \ \dd g' = [\kappa \star f](g^{-1}h),
\end{aligned}
\end{equation}
hence $\big[\kappa \star  \mathrm{Ind}_K^G \rho(g)f\big] = \mathrm{Ind}_K^G \sigma(g) [\kappa \star f]$.
\end{proof}

\begin{example}
Let us describe where ordinary CNNs fit in the present context. CNNs represent the case $G = \mathbb{Z}^2$ when $K = \{0\}$ is the trivial subgroup. The corresponding homogeneous space is $G/K = \mathbb{Z}^2/\{0\} = \mathbb{Z}^2$ and the quotient map $q : G \to G/K$ is thus the identity map on $\mathbb{Z}^2$. Its inverse, the identity map $\omega : G/K \to G$, is a globally defined gauge that eliminates the need for gauge equivariance, as we may choose to work exclusively in this one gauge. This is just a reflection of the fact that
\begin{equation}
G = \mathbb{Z}^2 = \mathbb{Z}/\{0\} \times \{0\} = G/K \times K,
\end{equation}
is (obviously) trivial as a principal bundle. Its associated bundles $E_\rho = \mathbb{Z}^2 \times_\rho V_\rho$ are also trivial: partly because the finite-dimensional $K$-representation $\sigma$ must be trivial, and partly because each equivalence class $[g,v]$ only contains a single representative. These reasons are, of course, due to the triviality of $K$.

This is not to say that the equivariant framework of Section \ref{subsec:ENN} is uninteresting when dealing with CNNs, or with GCNNs for other homogeneous spaces $\mathcal{M} = G/K$ with $K$ trivial. We saw in Sections \ref{subsec:hombundle}-\ref{subsec:indrep} how the homogeneity give rise to induced representations, which encode the global symmetry in both data points and feature maps. This is a useful perspective to have, and $G$-equivariant layers are interesting even when the bundles are trivial.

Triviality of the associated bundles, $E_\rho \simeq \mathbb{Z}^2 \times \mathbb{C}^{m}$ where $m = \dim V_\sigma$,\footnote{Recall that we focus on complex vector bundles, hence the use of $\mathbb{C}^m$ instead of $\mathbb{R}^m$.} implies that data points and feature maps are general square-integrable functions,
\begin{equation}
L^2(E_\rho) \simeq L^2(\mathbb{Z}^2;\rho) \simeq L^2(\mathbb{Z}^2,  \mathbb{C}^m),
\end{equation}
and are thereby extensions of compactly supported functions $f : \mathbb{Z}^2 \to \mathbb{C}^m$. This ties well into the discussion in Section \ref{subsec:CNN}. Convolutional layers \eqref{eq:convlayer} reduce to bounded linear operators $L^2(\mathbb{Z}^2,  \mathbb{C}^m) \to L^2(\mathbb{Z}^2, \mathbb{C}^n)$ and take the form
\begin{equation}\label{eq:convlayer_CNN}
(\kappa \star f)(x) = \sum_{y \in \mathbb{Z}^2} \kappa(y-x)f(y),
\end{equation}
as the Haar measure on $\mathbb{Z}^2$ is the counting measure. The kernel $\kappa : \mathbb{Z}^2 \to \mathrm{Hom}(\mathbb{C}^m,\mathbb{C}^n)$ is finitely supported in practice, so boundeness of \eqref{eq:convlayer_CNN} is ensured by  Lemma \ref{lemma:conv_bounded}.

Interestingly, all $\mathbb{Z}^2$-equivariant layers are convolutional layers; there are no other types of $\mathbb{Z}^2$-equivariant layers than \eqref{eq:convlayer_CNN}. This is a consequence of Theorem \ref{thm:main} and is proven in Corollary \ref{corr:LCA_discrete} below.
\end{example}

For more general groups $G$, it is no longer true that all $G$-equivariant layers are convolutional layers; we give an example of this fact in \mbox{Example \ref{ex:identity}}. Implementations of GCNNs, however, are usually based on convolutional layers, or on analogous layers in the Fourier domain. What consequences does the restriction to convolutional layers have for the expressivity of GCNNs? Can we tell whether a given $G$-equivariant layer is expressible as a convolutional layer? The answer to this last question, it turns out, requires the following notion of reproducing kernel Hilbert spaces.

\begin{definition}\label{def:RKHS}
Let $G$ be a group, let $V$ be a finite-dimensional normed vector space, and let $\mathcal{H}$ be a Hilbert space of functions $G \to V$. Then $\mathcal{H}$ is a \emph{reproducing kernel Hilbert space (RKHS)} if the evaluation operator
\begin{equation}
\mathcal{E}_g : \mathcal{H} \to V, \qquad f \mapsto f(g),
\end{equation}
is bounded for all $g \in G$. Moreover, by \emph{left-invariant RKH subspace} $\mathcal{H} \subseteq L^2(G,V)$ we mean a closed subspace that is both a RKHS and an invariant subspace for the left regular representation $\Lambda$ on $L^2(G,V)$.
\end{definition}

\begin{remark}
The term RKHS is typically reserved for the scalar case $V = \mathbb{C}$, when the evaluation operator is a linear functional. Our version would instead be dubbed \emph{vector-valued} RKHS. We see little benefit from distinguishing between these cases, however, so we use the term RKHS all-encompassingly.
\end{remark}

The name RKHS is due to the existence of a kernel-type function that reproduces all elements of $\mathcal{H}$. To see how, choose an orthonormal basis $e_1,\ldots,e_{\dim V} \in V$ and write elements $v \in V$ as linear combinations $v = \sum_i v^i e_i $. The projection $P_i(v)  = v^i$ onto the $i$'th component is always continuous, so the composition $\mathcal{E}_{g,i} := P_i \circ \mathcal{E}_g$ is a continuous linear functional
\begin{equation}
\mathcal{E}_g^i  : \mathcal{H} \to \mathbb{C}, \qquad  f \mapsto f^i(g),
\end{equation}
for all $g \in G$ and $i = 1,\ldots,\dim V$. By the Riesz representation theorem, there are elements $\varphi_{g,i} \in \mathcal{H}$ such that $f^i(g) = \mathcal{E}_{g,i}(f) = \langle f, \varphi_{g,i}\rangle$, hence
\begin{equation}\label{eq:linearcomb}
f(g) = \sum_{i=1}^{\dim V}  \langle f, \varphi_{g,i}\rangle e_i.
\end{equation}
Now, if $\mathcal{H} \subseteq L^2(G,V)$ is a left-invariant RKH subspace, expanding the functions $\varphi_{g,i}$ in the orthonormal basis, $\varphi_{g,i} = \sum_j \varphi_{g,i}^j e_j $, yields the formula
\begin{equation}\label{eq:RKHSintegral1}\begin{aligned}
f(g) &= \sum_i \left( \int_G \sum_j  \overline{\varphi_{g,i}^j(g')}f^j(g') \ \dd g'\right) e_i = \int_G \varphi_g^*(g') f(g') \ \dd g',
\end{aligned}\end{equation}
where $\varphi_g^*$ is the conjugate transpose of the matrix $(\varphi_g)_i^j = \varphi_{g,i}^j$. By left-invariance,
\begin{equation}\label{eq:RKHSintegral2}\begin{aligned}
f(g) &= \big( \Lambda(g^{-1}) f\big)(e) = \int_G \varphi_e^*(g^{-1}g') f(g') \ \dd g',
 \end{aligned}\end{equation}
hence $f \in \mathcal{H}$ is \emph{reproduced} by the operator-valued \emph{kernel} $\varphi_e : G \to \mathrm{Hom}(V)$.

\begin{remark}
The reproducing kernel $\varphi_e$ is unique and thus independent of the choice of basis in $V$. This follows from uniqueness in the Riesz representation theorem.
\end{remark}

It is now clear why left-invariant RKH subspaces of $L^2(G,V)$ are relevant when discussing convolutional layers, as the latter are given by integral operators similar to \eqref{eq:RKHSintegral2}. In order to show that an abstract $G$-equivariant layer $\phi : L^2(G;\rho) \to L^2(G;\sigma)$ can be written as a convolutional layer, it is almost necessary for it to act in a RKHS:

\begin{example}\label{ex:identity}
The identity operator $\phi : L^2(G;\sigma) \to L^2(G;\sigma)$ is clearly a $G$-equivariant layer regardless of $G$, $K$, $\sigma$, but it is only a convolutional layer if $L^2(G;\sigma)$ is a RKHS. This is because when $\phi$ is the identity, \eqref{eq:convlayer} becomes the reproducing property
\begin{equation}
f(g) = \int_G \kappa(g^{-1}g') f(g') \ \dd g', \qquad f \in L^2(G;\sigma).
\end{equation}
It follows that not every $G$-equivariant layer is a convolutional layer, because $L^2(G;\sigma)$ is not always a RKHS. When $\sigma$ is the trivial representation, for instance, $L^2(G;\sigma)$ reduces to $L^2(G)$ which is not a RKHS when $G$ is nondiscrete \cite[Theorem 2.42]{fuhr2005abstract}.
\end{example}

At this point, we know that global symmetry manifests itself in feature maps and data points through the induced representation, and we used this knowledge to define $G$-equivariant layers. We also defined convolutional layers and showed that these are special cases of $G$-equivariant layers, but the converse problem is much more subtle: When can a $G$-equivariant layer be expressed as a convolutional layer? The answer, as we have just seen, is directly related to the concept of RKHS and our next result makes this relation precise. It can be considered our main theorem.

\begin{theorem}\label{thm:main}
Let $G$ be a unimodular Lie group, let $K \leq G$ be a compact subgroup, and consider homogeneous vector bundles $E_\rho,E_\sigma$ over $\mathcal{M} = G/K$. Suppose that
\begin{equation}
\phi : L^2(G;\rho) \to L^2(G;\sigma),
\end{equation}
is a $G$-equivariant layer. If $\phi$ maps into a left-invariant RKH subspace $\mathcal{H} \subseteq L^2(G;\sigma)$, then $\phi$ is a convolutional layer.
\end{theorem}

\begin{proof}
Fix orthonormal bases in $V_\rho$, $V_\sigma$. For $i =1,\ldots,\dim\sigma$, consider the functionals
\begin{equation}\label{eq:eval_functional}
\mathcal{E}^i : L^2(G;\rho) \to \mathbb{C}, \qquad \mathcal{E}^i(f) = (\phi f)^i(e),
\end{equation}
composing $\phi$ with evaluation at the identity element $e \in G$ and projection onto the $i$'th component. As $\phi$ maps into a left-invariant RKH subspace $\mathcal{H} \subseteq L^2(G;\sigma)$, \eqref{eq:eval_functional} is a bounded linear functional: $|\mathcal{E}^i(f)| \leq \|(\phi f)(e)\|_\sigma \leq \| \phi f\|_{L^2(G;\sigma)} \leq \|\phi\| \|f\|_{L^2(G;\rho)}$. By the Riesz representation theorem, there is a unique $\varphi_i \in L^2(G;\rho)$ such that
\begin{equation}\label{eq:mainthm1}
\mathcal{E}^i(f) = \int_G \langle f(g), \varphi_i(g) \rangle_\rho \ \dd g  = \int_G \sum_{j=1}^{\dim\rho} f^i(g) \overline{ \varphi_i^j (g)} \ \dd g,
\end{equation}
and proceeding as in \eqref{eq:RKHSintegral1}-\eqref{eq:RKHSintegral2} with $\kappa := \varphi_e^*$ yields the desired relation
\begin{equation}\label{eq:thm_main_conv}
(\phi f)(g) = \int_G \kappa(g^{-1}g') f(g') \ \dd g'.
\end{equation}
\end{proof}

\begin{remark}
Theorem \ref{thm:main} is a generalization of \cite[Theorem 6.1]{cohen2018general}, which was proven under the assumption that $\phi$ is an integral operator $(\phi f)(g) = \int_G \kappa(g,g') f(g') \ \dd g'$.
\end{remark}

\begin{remark}
While Theorem \ref{thm:main} is similar in spirit to  \cite[Theorem 1]{kondor2018generalization}, there are also some clear differences. For example, we work with unimodular Lie groups whereas \cite{kondor2018generalization} use compact groups, but \cite[Theorem 1]{kondor2018generalization} is also stronger in this case as there is no criterion on the layer. Another difference is that \cite{kondor2018generalization} analyzes the whole network structure while we focus on individual layers. We also assume that the homogeneous space $G/K$ is the same before and after each layer, in constrast to \cite{kondor2018generalization}.

In the special case of single-layer networks with compact $G$, \cite[Theorem 1]{kondor2018generalization} states that any $G$-equivariant layer is a convolutional layer. Example \ref{ex:identity} seems to contradict this statement when $G$ is non-discrete compact. This conflict is possibly due to minor technical differences in the assumptions on layers and data points, but we have not identified the precise cause.
\end{remark}

We end this section with a result that could simplify the numerical computations of convolutional layers, as integrals over $G/K$ are sometimes easier to compute than integrals over $G$. For example when $G = SO(3)$, $K = SO(2)$, and $G/K \simeq S^2$. This result is similar to the generalized convolutions described in \cite[Section 4.1]{kondor2018generalization}

\begin{corollary}
Let $\phi : L^2(G;\rho) \to L^2(G;\sigma)$ be as in Theorem \ref{thm:main} and let $\kappa$ be the kernel of the resulting convolutional layer \eqref{eq:thm_main_conv}. Then
\begin{equation}\label{eq:quotient_int}
(\phi f)(g) = \int_{G/K} \kappa(g^{-1}x)f(x) \ \dd x.
\end{equation}
\end{corollary}
\begin{proof}
In the proof of Theorem \ref{thm:main}, we constructed the kernel $\kappa$ from the components of $\varphi_i \in L^2(G;\rho)$, and unitarity of $\rho$ clearly implies that the expression $\langle f(x) , \varphi_i(x)\rangle_\rho$ is well-defined. We may therefore use the unitary structure \eqref{eq:unitarystructure} to get the following relation for all component functions $(\phi f)^i$ and all $g \in G$:
\begin{alignat}{1}
(\phi f)^i(g) &=  \langle f, \mathrm{Ind}_K^G\rho(g) \varphi_i\rangle_{L^2(G;\rho)} = \langle s_f , s_{\mathrm{Ind}_K^G\rho(g) \varphi_i} \rangle_{L^2(E_\rho)}\\
&= \int_{G/K} \langle s_f(x) , s_{\mathrm{Ind}_K^G\rho(g) \varphi_i}(x) \rangle_x \ \dd x\\
&= \int_{G/K} \langle f(x) , \mathrm{Ind}_K^G\rho(g) \varphi_i(x)\rangle_\rho \ \dd x\\
&= \int_{G/K} \sum_{j=1}^{\dim \rho} \overline{\varphi_i^j(g^{-1}x)} f^j(x) \ \dd x.
\end{alignat}
We now obtain \eqref{eq:quotient_int} by reconstructing $\kappa$ from its components $\kappa_{ij} = \overline{\varphi_i^j}$.
\end{proof}

\subsection{RKHS and bandlimited functions}\label{subsec:RKHS}

The strength of Theorem \ref{thm:main} naturally depends on how common left-invariant RKH subspaces of $L^2(G;\sigma)$ are. Our analysis of $G$-equivariant layers would not be complete without a discussion on this topic.

Let us proceed by investigating when the component functions $f^i$ of $f \in L^2(G;\sigma)$ are contained in a left-invariant RKH subspace $\mathcal{H} \subset L^2(G)$; these subspaces have been fully character\-ized when the unimodular Lie group $G$ is of \emph{type I} \cite{carey1978group,fuhr2005abstract}.  The unitary equivalence \eqref{eq:unitary_equiv} then ensures that $A(\mathcal{H} \otimes V_\sigma) \subset L^2(G,V_\sigma)$ is a left-invariant RKH subspace, and so is the closed subspace
\begin{equation}\label{eq:RKHS_inclusion}
A(\mathcal{H} \otimes V_\sigma) \cap L^2(G;\sigma) \subset L^2(G;\sigma).
\end{equation}

\begin{remark}
Groups of type I are, in a sense, groups with manageable representation theory. They include the most common groups, such as all finite, discrete, compact, or abelian groups, the Euclidean groups and many other groups. In particular, there is a considerable overlap between type I groups and the unimodular Lie groups that we already consider. See \cite{folland2016course,fuhr2005abstract} for more details.
\end{remark}

\begin{remark}
While $\rho,\sigma$ still denote finite-dimensional unitary representations of $K$, we reserve the letter $\gamma$ for elements of the unitary dual $\widehat {G}$, i.e., the space of equivalence classes of unitary representations. Specific representatives of $\gamma$ are written as $(\pi_\gamma, V_\gamma)$, and note that $V_\gamma$ need not be finite-dimensional unless $G$ is compact. The unimodular Lie group $G$ is assumed to be of type I throughout this section.
\end{remark}

\begin{proposition}[{\cite[Proposition 2.40]{fuhr2005abstract}}]\label{thm:RKHS_and_conv}
Let $\mathcal{H} \subseteq L^2(G)$ be a left-invariant RKH subspace. The kernel $\varphi \in \mathcal{H}$ is then a self-adjoint convolution idempotent,\footnote{That is, $\varphi = \varphi * \varphi^* = \varphi^*$ where $\varphi^*(g) := \overline{\varphi(g^{-1})}$ denotes involution.} and
\begin{equation}
\mathcal{H} = L^2(G) * \varphi =  \left\{ f * \varphi \ \middle| \ f \in L^2(G)\right\} \subset C(G).
\end{equation}
Conversely, if $\varphi \in L^2(G)$ is a self-adjoint convolution idempotent, then $\mathcal{H} = L^2(G) * \varphi$ is a left-invariant RKH subspace of $L^2(G)$.
\end{proposition}

\begin{example}\label{ex:R}
Consider the real line $G = \mathbb{R}$ and suppose $\mathcal{H} \subseteq L^2(\mathbb{R})$ is a left-invariant RKH subspace with kernel $\varphi \in \mathcal{H}$. The calculation in \eqref{eq:RKHSintegral2} with $V = \mathbb{C}$ shows that, for all $f \in \mathcal{H}$,
\begin{equation}
f(x) = \int_{-\infty}^\infty \overline{\varphi(y-x)} f(y) \ \dd y = \langle f, \lambda(x)\varphi \rangle = (f * \varphi^*)(x).
\end{equation}
Since the regular representation $\lambda$ is continuous, $f$ must be continuous, so $\mathcal{H} \subset C(\mathbb{R})$. Setting $f = \varphi$ shows that the kernel is a self-adjoint convolution idempotent:
\begin{equation}
\varphi = \varphi * \varphi^* = (\varphi * \varphi^*)^* = \varphi^*.
\end{equation}
Combining the Plancherel transform on $L^2(\mathbb{R})$ (see Theorem \ref{thm:Plancherel} and Section \ref{subsubsec:abelian}) with the convolution theorem in Fourier analysis, we observe that, for all $f \in \mathcal{H}$,
\begin{equation}\label{eq:Rbandlimit}
\hat{f} = \widehat{f * \varphi} = \hat{f}\hat{\varphi}.
\end{equation}
In particular, $\hat{\varphi} = \hat{\varphi}^2$, so $\hat{\varphi}$ is the characteristic function $1_E$ on a subset $E \subset \widehat{\mathbb{R}} \simeq \mathbb{R}$. Inserting $\hat{\varphi} = 1_E$ in \eqref{eq:Rbandlimit} immediately tells us that $\supp(\hat{f}) \subset E$, so $\mathcal{H}$ is a space of bandlimited functions. Moreover, the set $E$ has finite Lebesgue measure according to the Plancherel theorem: $\mathrm{vol}(E) = \| 1_E\|_2^2 = \| \varphi\|_2^2 < \infty$.
\end{example}
This example illustrates that any measurable subset $E \subset \mathbb{R}$ with finite Lebesgue measure induces a left-invariant closed RKH subspace
\begin{equation}
\mathcal{H}_E = \left\{ f \in L^2(\mathbb{R}) \ \middle| \ \supp(\hat{f}) \subset E\right\} = L^2(\mathbb{R}) * \varphi_E,
\end{equation}
$\varphi_E$ being the inverse Plancherel transform of $1_E$ \cite[2.63-2.65]{fuhr2005abstract}. This relation between left-invariant RKH subspaces $\mathcal{H} \subseteq L^2(G)$ and bandlimited functions generalizes to unimodular Lie groups $G$ of type I, although the necessary harmonic analysis becomes significantly more advanced. Going into detail on this rather technical subject would distract from the topic at hand, so we refer curious readers to the relevant literature instead \cite{fuhr2005abstract}. Let us take the short route of stating a theorem on the direct integral decomposition of the left regular representation $\lambda$ and its commutant
\begin{equation}
\lambda(G)' = \left\{ T \in \mathcal{B}(L^2(G)) \ \middle| \ T \circ \lambda(g) = \lambda(g) \circ T \text{ for all } g \in G\right\},
\end{equation}
and discuss a few consequences of this decomposition, before restricting attention to two important cases where we can be more explicit: Abelian and compact groups.

\begin{definition}[{\cite[\S 3.5]{fuhr2005abstract}}]
The \emph{operator-valued Fourier transform} on $G$ maps each $f \in L^1(G)$ to the family $\mathcal{F}(f) = (\hat{f}(\gamma))_{\gamma \in \widehat{G}}$, where each $\hat{f}(\gamma) \in \mathcal{B}(V_\gamma)$ is a bounded operator given by the Bochner integral
\begin{equation}\label{eq:Fourier}
\hat{f}(\gamma) = \int_G f(g) \pi_\gamma(g) \ \dd g.
\end{equation}
\end{definition}

\begin{theorem}[{\cite[Theorem 3.48]{fuhr2005abstract}}]\label{thm:Plancherel}
There is a canonical \emph{Plancherel measure} $\nu$ for the unitary dual $\widehat{G}$ with the following properties:
\begin{enumerate}
\item[(a)] $\mathcal{F}$ extends to a unitary operator
\begin{equation}\label{eq:Plancherel}
\mathcal{P} : L^2(G) \to \int_{\widehat{G}}^\oplus V_\gamma \otimes V_\gamma \ \dd \nu(\gamma),
\end{equation}
called the \emph{Plancherel transform} of $G$.

\item[(b)] $\mathcal{P}$ implements the following unitary equivalences:
\begin{alignat}{1}
\lambda \simeq &\int_{\widehat{G}}^\oplus \pi_\gamma \otimes \Id \ \dd \nu(\gamma)\\
\lambda(G)' \simeq &\int_{\widehat{G}}^\oplus \Id \otimes \mathcal{B}(V_\gamma) \ \dd \nu(\gamma)
\end{alignat}
\end{enumerate}
\end{theorem}

Observe that if $\mathcal{H} \subseteq L^2(G)$ is a left-invariant closed subspace, then the projection $P : L^2(G) \to \mathcal{H}$ commutes with the left-regular representation and is thus an element of the commutant $\lambda(G)'$.  It therefore has a direct integral decomposition
\begin{equation}\label{eq:projection}
P = \int_{\widehat{G}}^\oplus \Id \otimes \hat{P}_\gamma \ \dd \nu(\gamma),
\end{equation}
where $\hat{P}_\gamma \in \mathcal{B}(V_\gamma)$ for each $\gamma \in \widehat{G}$.

\begin{theorem}[{\cite[Theorem 4.22, Proposition 2.40]{fuhr2005abstract}}]\label{thm:bandwidth}
Suppose that $\mathcal{H} \subseteq L^2(G)$ is a left-invariant closed subspace with equation \eqref{eq:projection} denoting the projection onto $\mathcal{H}$. Then $\mathcal{H}$ is a RKHS iff
\begin{equation}\label{eq:bandwidth}
\int_{\widehat{G}} \rank(\hat{P}_\gamma) \ \dd \nu(\gamma) < \infty.
\end{equation}
\end{theorem}

We interpret this theorem as a bandwidth restriction, similar to Example \ref{ex:R}. The integrand in \eqref{eq:bandwidth} is an integer-valued function on $\widehat{G}$, so the integral is finite only if the projection \eqref{eq:projection} is supported on a set $E \subseteq \widehat{G}$ of finite Plancherel measure,
\begin{equation}\label{eq:E}
\nu\Big( \underbrace{\overline{\left\{ \gamma \in \widehat{G} : \hat{P}_\gamma \neq 0\right\}}}_{E}\Big) = \int_{\widehat{G}} 1_E(\gamma) \ \dd \nu(\gamma) \leq \int_{\widehat{G}} \rank(\hat{P}_\gamma) \ \dd\nu(\gamma).
\end{equation}

That is, the left-invariant RKH subspaces $\mathcal{H} \subseteq L^2(G)$ are precisely those subspaces whose elements are bandlimited on a set $E \subseteq \widehat{G}$, in the sense that, for all $f \in \mathcal{H}$ and each equivalence class $\gamma \not\in E$,
\begin{equation}\label{eq:bandwidth2}
\hat{f}(\gamma) = \widehat{Pf}(\gamma) = \hat{f}(\gamma) \circ \hat{P}_\gamma = 0.
\end{equation}

\begin{remark}
The second equality in \eqref{eq:bandwidth2} is \cite[Corollary 4.17]{fuhr2005abstract}.
\end{remark}

Bandlimited functions are thus central to the theory of RKHS and, by extension, to the mathematical theory of GCNNs. Indeed, by extending the concept of bandwidth to feature maps, through \eqref{eq:RKHS_inclusion}, we obtain the following rephrasing of Theorem \ref{thm:main}.

\begin{corollary}\label{corr:main2}
Let $G$ be a unimodular Lie group of type I, let $K \leq G$ be a compact subgroup, and consider homogeneous vector bundles $E_\rho,E_\sigma$ over $G/K$. Suppose that
\begin{equation}
\phi : L^2(G;\rho) \to L^2(G;\sigma),
\end{equation}
is a $G$-equivariant layer. If $\phi$ maps into a space of bandlimited functions, then $\phi$ is a convolutional layer.
\end{corollary}

\begin{remark}
The relevance of bandwidth for convolutional layers has already been recognized in the case of azimuthally equivariant linear operators on $L^2(S^2)$ \cite{toft2021azimuthal}. In our setting, these operators translate to certain $G$-equivariant layers
\begin{equation}
\Phi : L^2(E_\rho) \to L^2(E_\sigma),
\end{equation}
when $G = SO(3)$, $K = SO(2)$, and $\rho,\sigma$ are the trivial representation. 
\end{remark}

\begin{remark}
Some implementations of GCNNs use Fourier transforms and a variant of the convolution theorem $\widehat{f_1 * f_2} = \hat{f}_1 \hat{f}_2$ to compute convolutional layers \cite{esteves20173d,kondor2018clebsch,toft2021azimuthal,weiler20183d}. Feature maps $f$ are then represented by their Fourier transform $\hat{f}$ which, for numerical reasons, is only approximated up to a finite bandlimit. That is, bandwidth is already being used in implementations.
\end{remark}

\subsubsection{Abelian groups}\label{subsubsec:abelian} The irreducible representations $\gamma \in \widehat{G}$ of any abelian group $G$ are 1-dimensional, and may thus be identified with their character $\chi_\gamma = \tr \pi_\gamma$. There are several useful consequences of this fact.

First, the unitary dual $\widehat{G}$ is now the set of continuous homomorphisms $\chi : G \to \mathbb{T}$, where $\mathbb{T}$ is the circle group. This is a locally compact group with respect to pointwise multiplication and, as $\gamma \in \widehat{G}$ is unitary, we may write $\chi_\gamma= e^{i\xi_\gamma }$ where $\xi_\gamma : G \to \mathbb{R}$. The Fourier transform then takes the more familiar form
\begin{equation}
\hat{f}(\gamma) = \int_G f(g) e^{-i\xi_\gamma(g)} \ \dd g,
\end{equation}
for $ f \in L^1(G) \cap L^2(G)$. Moreover, the Haar measure on $\widehat{G}$ can be made to coincide with the Plancherel measure such that \eqref{eq:Plancherel} becomes a unitary equivalence
\begin{equation}
\mathcal{P} : L^2(G) \to L^2(\widehat{G}).
\end{equation}

Another consequence of the fact that irreducible representations are 1-dimensional, is that the integrand in \eqref{eq:bandwidth} takes values in $\{0,1\}$ and \eqref{eq:E} becomes an equality. By the same arguments as in Example \ref{ex:R}, we see that the left-invariant RKH subspaces $\mathcal{H} \subset L^2(G)$ are the spaces of bandlimited functions, $\supp(\hat{f}) \subset E$, for subsets $E \subset \widehat{G}$ of finite Haar/Plancherel measure. Also, the kernel $\varphi_E \in \mathcal{H}$ is the inverse Plancherel transform of the characteristic function $1_E$.

\begin{corollary}\label{corr:LCA_discrete} If $G$ is a discrete and abelian group, then any $G$-equivariant layer is a convolutional layer.
\end{corollary}
\begin{proof}
Discrete groups are unimodular Lie group of type I, so we may use results from the current section. We note that the integral \eqref{eq:bandwidth} converges for any left-invariant, closed subspace $\mathcal{H} \subseteq L^2(G)$, as the unitary dual $\widehat{G}$ is compact when $G$ is discrete and abelian \cite[Proposition 3.1.5]{deitmar2014principles}, and because the integrand is bounded. Consequently, $\mathcal{H} = L^2(G)$ is itself a RKHS,\footnote{In fact, the kernel is simply the Kronecker delta $\varphi(g) = \delta_{g,e}$.} and the same is true for both $L^2(G,V_\sigma) \simeq L^2(G) \otimes V_\sigma$ and its closed, left-invariant subspace $L^2(G;\sigma)$, independently of $K \leq G$ and $(\sigma,V_\sigma)$. The result now follows from Theorem \ref{thm:main}.
\end{proof}

By setting $G = \mathbb{Z}^2$, Corollary \ref{corr:LCA_discrete} establishes that convolutional layers are the only possible translation equivariant layers in the ordinary CNN setting.

\subsubsection{Compact groups} When the group $G$ is compact, all irreducible representations are finite-dimensional. Furthermore, the unitary dual $\widehat{G}$ is discrete and the Plancherel measure on $\widehat{G}$ is simply the counting measure. For these reasons, the integral \eqref{eq:bandwidth} reduces to a discrete sum with finite summands, and converges iff $\hat{P}_\gamma = 0$ for all but finitely many $\gamma \in \widehat{G}$.

\begin{corollary}\label{corr:finiteG}
If $G$ is finite, then any $G$-equivariant layer is a convolutional layer.
\end{corollary}
\begin{proof}
When $G$ is a finite group, $\widehat{G}$ is also finite \cite[Proposition 5.27]{folland2016course} and the integral \eqref{eq:bandwidth} reduces to a finite sum. That is, $L^2(G)$ is a RKHS and the result now follows in the same way as Corollary \ref{corr:LCA_discrete}.
\end{proof}

\section{Discussion}\label{sec:discussion}

In this paper, we have investigated the mathematical foundations of $G$-equivariant convolutional neural networks (GCNNs), which are designed for deep learning tasks exhibiting global symmetry. We presented a basic framework for equivariant neural networks that include both gauge equivariant neural networks and GCNNs as special cases. We also demonstrated how GCNNs can be obtained from homogeneous vector bundles, when $G$ is a unimodular Lie group and $K \leq G$ is a compact subgroup.

In Theorem \ref{thm:main}, we gave a precise criterion for when a given $G$-equivariant layer is, in fact, a convolutional layer. This criterion uses reproducing kernel Hilbert spaces (RKHS) and cannot be circumvented, as shown in Example \ref{ex:identity}. After discussing the relation between RKHS and bandwidth, we were able to reformulate Theorem \ref{thm:main} to get an analogous bandlimit-criterion in Corollary \ref{corr:main2}. In Corollaries \ref{corr:LCA_discrete}-\ref{corr:finiteG}, we showed that the criterion is automatically satisfied when $G$ is discrete abelian or finite, hence all $G$-equivariant layers are convolutional layers for these groups.

One limitation of the current paper, compared to \cite{cohen2018general,kondor2018generalization}, is that the homogeneous space $G/K$ does not change between layers. This restriction was made in order to limit the scope of our analysis, and the same goes for our restriction to unimodular Lie groups. It would be interesting to go beyond these restrictions in the future.\\

I am grateful for the support from my research group: Oscar Carlsson, Jan Gerken, Hampus Linander, Fredrik Ohlsson, Christoffer Petersson, and last but not least, my advisor Daniel Persson. Thanks also to David Müller for the interesting discussions on gauge equivariance in physics. This work was supported by the Wallenberg AI, Autonomous Systems and Software Program (WASP) funded by the Knut and Alice Wallenberg Foundation. 

\bibliographystyle{plain}
\bibliography{references}

\begin{thebibliography}{53}
\providecommand{\natexlab}[1]{#1}
\providecommand{\url}[1]{\texttt{#1}}
\expandafter\ifx\csname urlstyle\endcsname\relax
  \providecommand{\doi}[1]{doi: #1}\else
  \providecommand{\doi}{doi: \begingroup \urlstyle{rm}\Url}\fi

\bibitem[Albawi et~al.(2017)Albawi, Mohammed, and
  Al-Zawi]{albawi2017understanding}
S.~Albawi, T.~A. Mohammed, and S.~Al-Zawi.
\newblock Understanding of a convolutional neural network.
\newblock In \emph{International Conference on Engineering and Technology
  (ICET)}. IEEE, 2017.

\bibitem[Alpaydin(2014)]{alpaydin2020introduction}
E.~Alpaydin.
\newblock \emph{Introduction to machine learning}.
\newblock MIT Press, 2014.

\bibitem[Boyda et~al.(2021)Boyda, Chernodub, Gerasimeniuk, Goy, Liubimov, and
  Molochkov]{boyda2021finding}
D.~L. Boyda, M.~N. Chernodub, N.~V. Gerasimeniuk, V.~A. Goy, S.~D. Liubimov,
  and A.~V. Molochkov.
\newblock Finding the deconfinement temperature in lattice yang-mills theories
  from outside the scaling window with machine learning.
\newblock \emph{Physical Review D}, 2021.

\bibitem[Brock et~al.(2021)Brock, De, Smith, and Simonyan]{brock2021high}
A.~Brock, S.~De, S.~L. Smith, and K.~Simonyan.
\newblock High-performance large-scale image recognition without normalization.
\newblock \emph{arXiv preprint arXiv:2102.06171}, 2021.

\bibitem[Bronstein et~al.(2017)Bronstein, Bruna, LeCun, Szlam, and
  Vandergheynst]{bronstein2017geometric}
M.~M. Bronstein, J.~Bruna, Y.~LeCun, A.~Szlam, and P.~Vandergheynst.
\newblock Geometric deep learning: going beyond euclidean data.
\newblock \emph{IEEE Signal Processing Magazine}, 2017.

\bibitem[Bronstein et~al.(2021)Bronstein, Bruna, Cohen, and
  Veli{\v{c}}kovi{\'c}]{bronstein2021geometric}
M.~M. Bronstein, J.~Bruna, T.~Cohen, and P.~Veli{\v{c}}kovi{\'c}.
\newblock Geometric deep learning: Grids, groups, graphs, geodesics, and
  gauges.
\newblock \emph{arXiv preprint arXiv:2104.13478}, 2021.

\bibitem[Cao et~al.(2020)Cao, Yan, He, and He]{cao2020comprehensive}
W.~Cao, Z.~Yan, Z.~He, and Z.~He.
\newblock A comprehensive survey on geometric deep learning.
\newblock \emph{IEEE Access}, 2020.

\bibitem[Carey(1978)]{carey1978group}
A.~L. Carey.
\newblock Group representations in reproducing kernel hilbert spaces.
\newblock \emph{Reports on Mathematical Physics}, 1978.

\bibitem[Cheng et~al.(2019)Cheng, Anagiannis, Weiler, de~Haan, Cohen, and
  Welling]{cheng2019covariance}
M.~Cheng, V.~Anagiannis, M.~Weiler, P.~de~Haan, T.~Cohen, and M.~Welling.
\newblock Covariance in physics and convolutional neural networks.
\newblock \emph{arXiv preprint arXiv:1906.02481}, 2019.

\bibitem[Cohen and Welling(2016)]{cohen2016group}
T.~Cohen and M.~Welling.
\newblock Group equivariant convolutional networks.
\newblock In \emph{International conference on machine learning}. PMLR, 2016.

\bibitem[Cohen et~al.(2018{\natexlab{a}})Cohen, Geiger, K{\"o}hler, and
  Welling]{cohen2018spherical}
T.~Cohen, M.~Geiger, J.~K{\"o}hler, and M.~Welling.
\newblock Spherical cnns.
\newblock \emph{arXiv preprint arXiv:1801.10130}, 2018{\natexlab{a}}.

\bibitem[Cohen et~al.(2018{\natexlab{b}})Cohen, Geiger, and
  Weiler]{cohen2018general}
T.~Cohen, M.~Geiger, and M.~Weiler.
\newblock A general theory of equivariant cnns on homogeneous spaces.
\newblock \emph{arXiv preprint arXiv:1906.02481}, 2018{\natexlab{b}}.

\bibitem[Cohen et~al.(2019)Cohen, Weiler, Kicanaoglu, and
  Welling]{cohen2019gauge}
T.~Cohen, M.~Weiler, B.~Kicanaoglu, and M.~Welling.
\newblock Gauge equivariant convolutional networks and the icosahedral cnn.
\newblock In \emph{International Conference on Machine Learning}. PMLR, 2019.

\bibitem[Deitmar and Echterhoff(2014)]{deitmar2014principles}
A.~Deitmar and S.~Echterhoff.
\newblock \emph{Principles of harmonic analysis}.
\newblock Springer International Publishing, 2014.

\bibitem[Esteves(2020)]{esteves2020theoretical}
C.~Esteves.
\newblock Theoretical aspects of group equivariant neural networks.
\newblock \emph{arXiv preprint arXiv:2004.05154}, 2020.

\bibitem[Esteves et~al.(2017)Esteves, Allen-Blanchette, Makadia, and
  Daniilidis]{esteves20173d}
C.~Esteves, C.e Allen-Blanchette, A.~Makadia, and K.~Daniilidis.
\newblock {{3D}} object classification and retrieval with spherical cnns.
\newblock \emph{arXiv preprint arXiv:1711.06721}, 2017.

\bibitem[Esteves et~al.(2018)Esteves, Allen-Blanchette, Makadia, and
  Daniilidis]{esteves2018learning}
C.~Esteves, C.~Allen-Blanchette, A.~Makadia, and K.~Daniilidis.
\newblock Learning so(3) equivariant representations with spherical cnns.
\newblock In \emph{Proceedings of the European Conference on Computer Vision
  (ECCV)}, 2018.

\bibitem[Favoni et~al.(2020)Favoni, Ipp, M{\"u}ller, and
  Schuh]{favoni2020lattice}
M.~Favoni, A.~Ipp, D.~M{\"u}ller, and D.~Schuh.
\newblock Lattice gauge equivariant convolutional neural networks.
\newblock \emph{arXiv preprint arXiv:2012.12901}, 2020.

\bibitem[Folland(2016)]{folland2016course}
G.~B. Folland.
\newblock \emph{A course in abstract harmonic analysis}.
\newblock CRC press, 2016.

\bibitem[F{\"u}hr(2005)]{fuhr2005abstract}
H.~F{\"u}hr.
\newblock \emph{Abstract harmonic analysis of continuous wavelet transforms}.
\newblock Springer Science \& Business Media, 2005.

\bibitem[Fukushima and Miyake(1982)]{fukushima1982neocognitron}
K.~Fukushima and S.~Miyake.
\newblock Neocognitron: A self-organizing neural network model for a mechanism
  of visual pattern recognition.
\newblock In \emph{Competition and cooperation in neural nets}. Springer, 1982.

\bibitem[Gattringer and Lang(2009)]{gattringer2009quantum}
C.~Gattringer and C.~Lang.
\newblock \emph{Quantum chromodynamics on the lattice: an introductory
  presentation}.
\newblock Springer Science \& Business Media, 2009.

\bibitem[Gerken et~al.(Forthcoming)Gerken, Aronsson, Carlsson, Linander,
  Ohlsson, Petersson, and Persson]{gerken2021geometric}
J.~E. Gerken, J.~Aronsson, O.~Carlsson, H.~Linander, F.~Ohlsson, C.~Petersson,
  and D.~Persson.
\newblock Geometric deep learning and equivariant neural networks.
\newblock Forthcoming.

\bibitem[Goodfellow et~al.(2016)Goodfellow, Bengio, and
  Courville]{Goodfellow-et-al-2016}
I.~Goodfellow, Y.~Bengio, and A.~Courville.
\newblock \emph{Deep learning}.
\newblock MIT Press, 2016.
\newblock URL \url{http://www.deeplearningbook.org}.

\bibitem[Jia et~al.(2021)Jia, Yang, Xia, Chen, Parekh, Pham, Le, Sung, Li, and
  Duerig]{jia2021scaling}
C.~Jia, Y.~Yang, Y.~Xia, Y-T. Chen, Z.~Parekh, H.~Pham, Q.V. Le, Y.~Sung,
  Z.~Li, and T.~Duerig.
\newblock Scaling up visual and vision-language representation learning with
  noisy text supervision.
\newblock \emph{arXiv preprint arXiv:2102.05918}, 2021.

\bibitem[Kol{\'a}r et~al.(2013)Kol{\'a}r, Michor, and
  Slov{\'a}k]{kolar2013natural}
I.~Kol{\'a}r, P.~W. Michor, and J.~Slov{\'a}k.
\newblock \emph{Natural operations in differential geometry}.
\newblock Springer Science \& Business Media, 2013.

\bibitem[Koller et~al.(2018)Koller, Zargaran, Ney, and Bowden]{koller2018deep}
O.~Koller, S.~Zargaran, H.~Ney, and R.~Bowden.
\newblock Deep sign: Enabling robust statistical continuous sign language
  recognition via hybrid cnn-hmms.
\newblock \emph{International Journal of Computer Vision}, 2018.

\bibitem[Kondor and Trivedi(2018)]{kondor2018generalization}
R.~Kondor and S.~Trivedi.
\newblock On the generalization of equivariance and convolution in neural
  networks to the action of compact groups.
\newblock In \emph{International Conference on Machine Learning}. PMLR, 2018.

\bibitem[Kondor et~al.(2018)Kondor, Lin, and Trivedi]{kondor2018clebsch}
R.~Kondor, Z.~Lin, and S.~Trivedi.
\newblock Clebsch-gordan nets: a fully fourier space spherical convolutional
  neural network.
\newblock \emph{arXiv preprint arXiv:1806.09231}, 2018.

\bibitem[Krizhevsky et~al.(2012)Krizhevsky, Sutskever, and
  Hinton]{krizhevsky2012imagenet}
Alex Krizhevsky, Ilya Sutskever, and Geoffrey~E Hinton.
\newblock Imagenet classification with deep convolutional neural networks.
\newblock \emph{Advances in neural information processing systems}, 2012.

\bibitem[Lang and Weiler(2020)]{lang2020wigner}
L.~Lang and M.~Weiler.
\newblock A wigner-eckart theorem for group equivariant convolution kernels.
\newblock \emph{arXiv preprint arXiv:2010.10952}, 2020.

\bibitem[Lawrence et~al.(1997)Lawrence, Giles, Tsoi, and
  Back]{lawrence1997face}
S.~Lawrence, C.L. Giles, A.C. Tsoi, and A.D. Back.
\newblock Face recognition: A convolutional neural-network approach.
\newblock \emph{IEEE transactions on neural networks}, 1997.

\bibitem[LeCun et~al.(1995)LeCun, Jackel, Bottou, Cortes, Denker, Drucker,
  Guyon, Muller, Sackinger, Simard, et~al.]{lecun1995learning}
Y.~LeCun, L.D. Jackel, L.~Bottou, C.~Cortes, J.S. Denker, H.~Drucker, I.~Guyon,
  U.A. Muller, E.~Sackinger, P.~Simard, et~al.
\newblock Learning algorithms for classification: A comparison on handwritten
  digit recognition.
\newblock \emph{Neural networks: the statistical mechanics perspective}, 1995.

\bibitem[Lee(2013)]{lee2013smooth}
J.~M. Lee.
\newblock \emph{Introduction to smooth manifolds}.
\newblock Springer, 2013.

\bibitem[Li et~al.(2019)Li, Bi, and Lee]{li2019discrete}
J.~Li, Y.~Bi, and G.H. Lee.
\newblock Discrete rotation equivariance for point cloud recognition.
\newblock In \emph{2019 International Conference on Robotics and Automation
  (ICRA)}. IEEE, 2019.

\bibitem[Luo et~al.(2021)Luo, Carleo, Clark, and Stokes]{luo2021gauge}
D.~Luo, G.~Carleo, B.~Clark, and J.~Stokes.
\newblock Gauge equivariant neural networks for quantum lattice gauge theories.
\newblock \emph{Bulletin of the American Physical Society}, 2021.

\bibitem[Mehta et~al.(2017)Mehta, Sridhar, Sotnychenko, Rhodin, Shafiei,
  Seidel, Xu, Casas, and Theobalt]{mehta2017vnect}
D.~Mehta, S.~Sridhar, O.~Sotnychenko, H.~Rhodin, M.~Shafiei, H-P. Seidel,
  W.~Xu, D.~Casas, and C.~Theobalt.
\newblock Vnect: Real-time 3d human pose estimation with a single rgb camera.
\newblock \emph{ACM Transactions on Graphics (TOG)}, 2017.

\bibitem[Monti et~al.(2017)Monti, Boscaini, Masci, Rodola, Svoboda, and
  Bronstein]{monti2017geometric}
F.~Monti, D.~Boscaini, J.~Masci, E.~Rodola, J.~Svoboda, and M.~M. Bronstein.
\newblock Geometric deep learning on graphs and manifolds using mixture model
  cnns.
\newblock In \emph{Proceedings of the IEEE conference on computer vision and
  pattern recognition}, 2017.

\bibitem[M{\"u}ller(2019)]{Muller2019YangMills}
D.~M{\"u}ller.
\newblock Yang-mills theory, lattice gauge theory and simulations, 2019.
\newblock URL
  \url{https://www.jku.at/en/institut-of-analysis/research/guest-lecture-series/}.

\bibitem[Nakahara(2003)]{nakahara2003geometry}
M.~Nakahara.
\newblock \emph{Geometry, topology and physics}.
\newblock CRC press, 2003.

\bibitem[Peng et~al.(2017)Peng, Zhang, Yu, Luo, and Sun]{peng2017large}
C.~Peng, X.~Zhang, G.~Yu, G.~Luo, and J.~Sun.
\newblock Large kernel matters - improve semantic segmentation by global
  convolutional network.
\newblock In \emph{Proceedings of the IEEE conference on computer vision and
  pattern recognition}, 2017.

\bibitem[Song et~al.(2019)Song, Huang, and Ruan]{song2019abstractive}
S.~Song, H.~Huang, and T.~Ruan.
\newblock Abstractive text summarization using lstm-cnn based deep learning.
\newblock \emph{Multimedia Tools and Applications}, 2019.

\bibitem[Steenrod(1960)]{steenrod1960topology}
N.~Steenrod.
\newblock \emph{The topology of fibre bundles}.
\newblock Princeton University Press, 1960.

\bibitem[Tan and Le(2021)]{tan2021efficientnetv2}
M.~Tan and Q.V. Le.
\newblock Efficientnetv2: Smaller models and faster training.
\newblock \emph{arXiv preprint arXiv:2104.00298}, 2021.

\bibitem[Toft et~al.(2021)Toft, Bökman, and Kahl]{toft2021azimuthal}
C.~Toft, G.~Bökman, and F.~Kahl.
\newblock Azimuthal rotational equivariance in spherical cnns.
\newblock 2021.
\newblock URL \url{https://openreview.net/forum?id=sp3Z1jiS2vn}.

\bibitem[Veeling et~al.(2018)Veeling, Linmans, Winkens, Cohen, and
  Welling]{veeling2018rotation}
B.S. Veeling, J.~Linmans, J.~Winkens, T.~Cohen, and M.~Welling.
\newblock Rotation equivariant cnns for digital pathology.
\newblock In \emph{International Conference on Medical image computing and
  computer-assisted intervention}. Springer, 2018.

\bibitem[Wallach(2018)]{wallach2018harmonic}
N.~R. Wallach.
\newblock \emph{Harmonic analysis on homogeneous spaces}.
\newblock M. Dekker, 2018.

\bibitem[Weiler et~al.(2018)Weiler, Geiger, Welling, Boomsma, and
  Cohen]{weiler20183d}
M.~Weiler, M.~Geiger, M.~Welling, W.~Boomsma, and T.~Cohen.
\newblock {{3D Steerable CNNs}}: {{Learning Rotationally Equivariant Features}}
  in {{Volumetric Data}}.
\newblock \emph{arXiv preprint arXiv:1807.02547}, 2018.

\bibitem[Worrall and Brostow(2018)]{worrall2018cubenet}
D.~Worrall and G.~Brostow.
\newblock Cubenet: Equivariance to 3d rotation and translation.
\newblock In \emph{Proceedings of the European Conference on Computer Vision
  (ECCV)}, 2018.

\bibitem[Wu et~al.(2021)Wu, Xiao, Codella, Liu, Dai, Yuan, and
  Zhang]{wu2021cvt}
H.~Wu, B.~Xiao, N.~Codella, M.~Liu, X.~Dai, L.~Yuan, and L.~Zhang.
\newblock Cvt: Introducing convolutions to vision transformers.
\newblock \emph{arXiv preprint arXiv:2103.15808}, 2021.

\bibitem[Wu(2017)]{wu2017introduction}
J.~Wu.
\newblock Introduction to convolutional neural networks.
\newblock \emph{National Key Lab for Novel Software Technology. Nanjing
  University. China}, 2017.

\bibitem[Xu et~al.(2017)Xu, Jia, Wang, Ai, Zhang, Lai, Eric, and
  Chang]{xu2017large}
Y.~Xu, Z.~Jia, L-B. Wang, Y.~Ai, F.~Zhang, M.~Lai, I~Eric, and C.~Chang.
\newblock Large scale tissue histopathology image classification, segmentation,
  and visualization via deep convolutional activation features.
\newblock \emph{BMC bioinformatics}, 2017.

\bibitem[Zhang et~al.(2018)Zhang, Qiu, Yao, Liu, and Mei]{zhang2018fully}
Y.~Zhang, Z.~Qiu, T.~Yao, D.~Liu, and T.~Mei.
\newblock Fully convolutional adaptation networks for semantic segmentation.
\newblock In \emph{Proceedings of the IEEE Conference on Computer Vision and
  Pattern Recognition}, 2018.

\end{thebibliography}

\end{document}